\documentclass{article}
\PassOptionsToPackage{numbers, sort}{nonatbib}
\usepackage[preprint]{neurips_2023}
\usepackage{graphicx} 
\usepackage[utf8]{inputenc} 
\usepackage[T1]{fontenc}    
\usepackage{hyperref}       
\usepackage{url}            
\usepackage{booktabs}       
\usepackage{amsfonts}       
\usepackage{nicefrac}       
\usepackage{microtype}      
\usepackage{xcolor}         
\usepackage{subfigure}
\usepackage{graphicx}
\usepackage{amsmath}
\usepackage{amsfonts} 
\usepackage{mathrsfs}
\usepackage{enumerate}
\usepackage{amsthm}
\usepackage{wrapfig}
\usepackage{colortbl}
\usepackage{caption}

\newtheorem{definition}{Definition}
\newtheorem{proposition}{Proposition}
\newtheorem{theorem}{Theorem}

\title{Prior Bilinear Based Models for Knowledge Graph Completion}

\author{%
  Jiayi Li$^{12}$\thanks{Equal contribution.
  $\dagger$ Corresponding author: yang.yujiu@sz.tsinghua.edu.cn.} , 
  Ruilin Luo$^{1\ast}$,
  \textbf{Jiaqi Sun$^{3}$},
  \textbf{Jing Xiao$^{4}$},
  \textbf{Yujiu Yang$^{1\dagger}$} \\
  \\$^{1}$Tsinghua Shenzhen International Graduate School, Tsinghua University
  \\$^{2}$Baidu Inc.
  \\$^{3}$Carnegie Mellon University
  \\$^{4}$Pingan Group
  \\ \texttt{\{lijy20, lrl23\}@mails.tsinghua.edu.cn}
}

\begin{document}

\maketitle
\begin{abstract}
Bilinear based models are powerful and widely used approaches for Knowledge Graphs Completion (KGC).
Although bilinear based models have achieved significant advances, these studies mainly concentrate on posterior properties (based on evidence, e.g. symmetry pattern) while neglecting the prior properties. In this paper, we find a prior property named "the law of identity" that cannot be captured by bilinear based models, which hinders them from comprehensively modeling the characteristics of KGs. To address this issue, we introduce a solution called \underline{Uni}t Ball \underline{Bi}linear Model (UniBi). This model not only achieves theoretical superiority but also offers enhanced interpretability and performance by minimizing ineffective learning through minimal constraints. Experiments demonstrate that UniBi models the prior property and verify its interpretability and performance.

\end{abstract}

\section{Introduction}
\label{sec:intro}
Knowledge Graphs (KGs) store human knowledge in the form of triple $(h, r, t)$, which represents a relation $r$ between a head entity $h$ and a tail entity $t$~\citep{survey1}. KGs benefit a lots of downstream tasks and applications, e.g., recommender system~\citep{recommend}, dialogue system~\citep{dialogue} and question answering~\citep{QA}. Since actual KGs are usually incomplete, researchers are interested in predicting missing links to complete them, termed as Knowledge Graph Completion (KGC). As a common solution, Knowledge Graph Embedding (KGE) completes KGs by learning low-dimensional representations of entities and relations.

As one typical category of KGE, bilinear based models have achieved great advances~\citep{complex, cp, analogy,rescal}, 
these work only focus on posterior properties, which are based on evidence of triples, such as relational patterns~\cite{dihedral, analogy} and complex relations~\cite{quatde, tang-etal-2020-orthogonal}. For example, we treat a relation as symmetric based on the observation that $(h, r, t)$ and $(t, r, h)$ are co-occurring.
Here, \textbf{we ask does a prior property exist?} 

Our answer is \textbf{the law of identity} in Logic~\cite{logic}, which means that everything is only identical to itself. In KGs, this rule implies that not only the representations of entities are different but also the representation of \textit{identity} should be unique so that it can be decided without any facts, or \textit{a priori}. 
However, we find that the uniqueness of \textit{identity} have not been captured by previous bilinear based model, which prevents them to fully capture the properties of KGs.





To present the problem more clearly, we first need to introduce some notation. A model with a score function $s(h,r,t)$ can model the uniqueness of \textit{identity} means that $\forall h \neq t, s(h, r, h) > s(h, r, t)$ holds if and only if $r$ is \textit{identity} and its universal representation is unique. In addition, the score function $s(\cdot)$ of bilinear based model is $\mathbf{h}^\top \mathbf{R} \mathbf{t}$, where $\mathbf{h, R, t}$ are the representations of $h$, $r$, and $t$.

In terms of such uniqueness, bilinear based models have two flaws. On the one hand, Fig. \ref{fig:case 1} demonstrates $\mathbf{e}_1^\top \mathbf{I}\mathbf{e}_1 < \mathbf{e}_1^\top \mathbf{I}\mathbf{e}_2$, which means that the relation matrices per se do not model $identity$ perfectly. On the other hand, Fig. \ref{fig:case 2} shows even if a matrix, e.g. $\mathbf{I}$, does. Its scaled one $k\mathbf{I}$ can also model \textit{identity} and thus breaks the uniqueness.

Obviously, modeling this property requires both entities and relations to be restricted, which reduces expressiveness. Yet, we make this cost negligible by minimizing the constraints, one per entity or relation, while modeling the desired property. To be specific, we normalize the vectors of the entities and the spectral radius of the matrices of the relations to $1$. Since the model captures entities in a unit ball as shown in Fig. \ref{fig:unit ball}, we name it \underline{Uni}t Ball \underline{Bi}linear Model (UniBi)



In addition to the theoretical superiority, UniBi is more powerful and interpretable since modeling identity uniquely requires normalizing the scales that barely contain any useful knowledge.
On the one hand, scale normalization prevents ineffective learning on scales and makes UniBi focus more on useful knowledge. On the other hand, it shows the relationship between the complexity of relations and the singular values in the matrices used to represent them.


Experiments verify that UniBi models the law of identity with improvement on performance and interpretability. Therefore, UniBi is a prior KGE model and a potential paradigm for bilinear based models.




\begin{figure}[t]
    \centering
    \subfigure[]{
        \includegraphics[width=3cm]{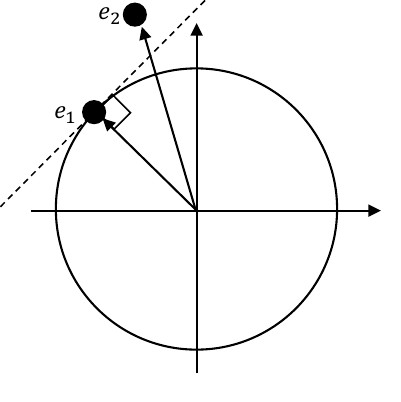}
        \label{fig:case 1}
    }
    \quad\quad
    \subfigure[]{
        \includegraphics[width=3cm]{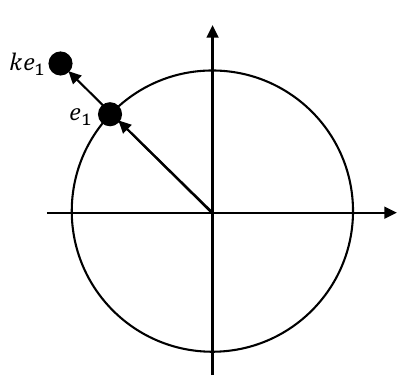}
        \label{fig:case 2}
    }
    \quad\quad
    \subfigure[]{
        \includegraphics[width=3cm]{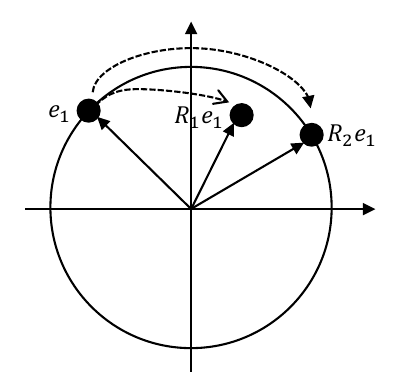}
        \label{fig:unit ball}
    }
    \caption{The flaws of bilinear based models and our solution in terms of modeling the uniqueness of \textit{identity}. (a) Identity matrix fails to model \textit{identity}. (b) Scaled identity matrix could also model \textit{identity}. (c) An illustration of UniBi. All entities are embedded in the unit sphere and stay in the unit ball after relation specific transformations.}
    \label{fig:my_label}
\end{figure}
\section{Preliminaries}
\label{sec:prelim}
\subsection{Background}
A Knowledge Graph $\mathcal{K}$ is a set that contains the facts about sets of entities $\mathcal{E}$ and relations $\mathcal{R}$. Each fact is stored by a triple $(e_i, r_j, e_k) \in \mathcal{E} \times \mathcal{R} \times \mathcal{E}$ where $e_i$ and $r_j$ denote the $i$ -th entity and the $j$ -th relation, respectively.

KGE learns embeddings for each entity and relation via a score function $s:\mathcal{E}\times\mathcal{R}\times\mathcal{E}\to \mathbb{R}$. To verify the performance of a KGE method, $\mathcal{K}$ is first divided into $\mathcal{K}_{train}$ and $\mathcal{K}_{test}$. Then, the method is trained on $\mathcal{K}_{train}$ to learn the embeddings $\mathbf{e}$ and $\mathbf{r}$ (or $\mathbf{R}$) for each entity $e\in\mathcal{E}$ and relation $r\in\mathcal{R}$. Finally, the model is expected to give a higher rank if $(e_i, h_j, e_k) \in \mathcal{K}$ while a lower rank if $(e_i, h_j, e_k) \notin \mathcal{K}$, for each query $(e_i, h_j, ?)$ from $\mathcal{K}_{test}$ and the candidate entity $e_k \in \mathcal{E}$.

In addition to the above tail prediction, models are also required to test on head prediction conversely.
We transform it to tail prediction by introducing reciprocal relations following \citet{n3}.

In logic, posterior and prior knowledge are distinguished by whether it holds without evidence.
In KGs, we say relational patterns and complex relations are posterior because we cannot assign them to any of the relations without a triplet or other information. For example, we need to observe that $(h, r, t)$ and $(t, r, h)$ are co-occurring to decide a relation $r$ has the symmetry pattern or observe $(h, r, t_1)$ and $(h,r,t_2)$ to conclude $r$ is a 1-N relation.

We claim that the law of identity is a prior property that is true for all entities in any KGs. To be specific, the law of identity is one of the basic prior rule in Logic~\cite{logic}, it means that everything is identical to itself, or $\forall x, x=x$. And this law corresponds to the definition of entity\footnote{Cambridge Dictionary defines it as "Something that exists apart from other things, having its own independent existence".}.


\subsection{Other Notations}
We utilize $\hat{\mathbb{E}}$ and $\hat{\mathbb{R}}$ to denote the set of all possible representations of entities and relations. And we use $\mathbf{e} \in \hat{\mathbb{E}}$ and $\mathbf{R}\in \hat{\mathbb{R}}$ to denote the embedding vector of the entity $e$ and the transformation matrix specific to the relation $\mathbf{R}$. Furthermore, we use $\|\cdot\|$ to denote the L2 norm of the vectors, $\|\cdot\|_F$ and $\rho(\cdot)$ to represent the Frobenius norm and the spectral radius of a matrix. 


In this paper, we focus on $n$-dimensional real space $\mathbb{R}^n$, which means $\mathbb{\hat{E}} \subseteq \mathbb{R}^n$ and $\mathbb{\hat{R}} \subseteq \mathbb{R}^{n\times n}$. 
We also consider real vector spaces whose vectors are complex $\mathbb{C}^n$ or hypercomplex space $\mathbb{H}^n$, since they are isomorphic to $\mathbb{R}^{2n}$ or $\mathbb{R}^{4n}$.

\section{Related Work}
\label{sec:related work}

Previous work mainly handle two kind of posterior properties, namely relational patterns and complex relations. On the one hand, relational patterns are the intrinsic properties of relations, and is formally introduced by ComplEx~\cite{complex}. Base on this, RotatE~\cite{sun2018rotate} proposes composition pattern, Analogy~\cite{analogy} introduces analogy pattern, or commutative pattern, and Dihedral~\cite{dihedral} adds non-commutative pattern. On the other hand, complex relations are the extrinsic properties of relation, and is introduced in TransH~\cite{transh} to denote the relations that are not 1-1, or 1-N, N-1, N-N.

Beyond posterior properties, previous work of KGE can be roughly divided into the following three categories: distance, bilinear and others. 

Distance based models choose Euclidean distance for their score functions. TransE~\citep{bordes2013translating} inspired by Word2Vec~\citep{Mikolov2013DistributedRO} in Natural Language Processing proposes the first distance based model, which uses translation as the linear transformation $s(h,r,t) = -\|\mathbf{h} + \mathbf{r} - \mathbf{t}\|$. TransH~\citep{transr} and TransR~\citep{transr} find that TransE difficult to handle complex relations and thus apply linear projections before translation. Apart from translation, RotatE~\citep{sun2018rotate} first introduces rotation as the transformation. RotE~\citep{rote} further combines translation and rotation. Some works also introduce hyperbolic spaces~\citep{murp,rote,mix}. 

In contrast, bilinear based models have score functions in the bilinear form $s(h,r,t) = \mathbf{h^\top Rt}$. RESCAL~\citep{rescal} is the first bilinear based model whose relation matrices are unconstrained. Although RESCAL is expressive, it contains too many parameters and tends to overfitting. DistMult~\citep{distmult} simplifies these matrices into diagonal ones. ComplEx~\citep{complex} further introduces complex values to model the skew-symmetry pattern. Analogy~\citep{analogy} uses block-diagonal to model the analogical pattern and subsumes DistMult, ComplEx, and HolE~\citep{hole}. Moreover, QuatE~\citep{quate} extends complex values to quaternion and GeomE~\citep{geometric} utilizes geometric algebra to subsume all these models.

In addition, other works using black-box networks~\citep{conve,convkb,kgbert,rgcn,rghat} or additional information~\citep{text,type} are beyond the scope of this paper.


\section{Method}
In this section, we first discuss prior property, and give the condition for bilinear based models to model the law of identity in Section \ref{sec:math condition}. We then propose a model named UniBi that satisfies it with least constraint in Section \ref{sec:unibi} and an efficient modeling for UniBi in Section \ref{sec:efficient modeling}. In addition, we also discuss its improvement on performance and interpretability via scale normalization in Section \ref{sec:properties}.
\label{sec:method}
\subsection{Prior Property and Identity Relation}
\label{sec:math condition}

In Section \ref{sec:prelim}, we have shown that the law of identity means $\forall x, x=x$. We consider how entities and \textit{identity} relation, i.e. $x$ and $=$, are embedded, and the fact the embedding of \textit{identity} should be specified \textit{a priori}, we find that this is equivalent to the following definition:

\begin{definition}
\label{def:identity in KGE}
A KG model can model the law of identity, which means that the embeddings of entities are different and the embeddings of the identity relation is unique.  
\end{definition}

We notice the differences in the embedding of entities are met effortlessly, yet 
the uniqueness of identity cannot be modeled by bilinear based models. We demonstrate two cases in which bilinear based models violate it. 
On the one hand, Fig. \ref{fig:case 1} demonstrates $\mathbf{e}_1^\top \mathbf{I}\mathbf{e}_1 < \mathbf{e}_1^\top \mathbf{I}\mathbf{e}_2$, which means that the matrix of a relation per se is not guaranteed for modeling \textit{identity}. On the other hand, Fig. \ref{fig:case 2} shows even if a matrix, e.g., $\mathbf{I}$, does. Its scaled one $k\mathbf{I}$ where $k > 0,k\neq 1$ can also model \textit{identity}, which contradicts the quantification of uniqueness. Therefore, we give a formal definition based on definition \ref{def:identity in KGE} as following to investigate how to modify bilinear based models to model this uniqueness and the law of identity.

\begin{definition}
\label{def:identity}
A bilinear model can model the law of identity means: 
\begin{equation}
\label{equ:identity}
    \exists!\ \mathbf{R}\in\mathbb{\hat{R}}, \forall \mathbf{h,t}\in\mathbb{\hat{E}}, \mathbf{h}\neq\mathbf{t}, \  \mathbf{h^\top R h} > \mathbf{h^\top R t},
\end{equation}
where $\exists !$ is the uniqueness quantification. 
\end{definition}


\subsection{Unit Ball Bilinear Model}
\label{sec:unibi}

From the above examples, it is easy to see that modeling the law of identity requires both entities and relations be restricted, which will reduce expressiveness. To solve this dilemma, we make the cost negligible by minimizing the constraints, one per entity or relation, while modeling the desired property.

To be specific, we normalize the vectors of the entities and the spectral radius of the matrices of the relations to $1$ by setting $\mathbb{\hat{E}} = \{\mathbf{e}\,|\,\|\mathbf{e}\| = 1, \mathbf{e} \in \mathbb{R}^n\}$ and $\mathbb{\hat{R}} = \{\mathbf{R}\,|\,\rho(\mathbf{R}) = 1,\mathbf{R}\in\mathbb{R}^{n\times n}  \}$. We name the proposed model as \underline{Uni}t Ball \underline{Bi}linear Model (UniBi), since it captures entities in a unit ball as shown in Fig. \ref{fig:unit ball}. The score function of UniBi is\footnote{We present in more detail the similarities and differences between our constraints and those of our predecessors in the Appendix \ref{app: comparison to previous restrictions}}:
\begin{equation}
\label{equ:unibi:constrain}
    s(h,r,t) = \mathbf{h^\top Rt},\ \|\mathbf{h}\|, \|\mathbf{t}\|=1,\rho(\mathbf{R}) = 1.
\end{equation}

We then have the following theorem.
\begin{theorem}
\label{theorem: can model}
UniBi is capable to model the law of identity in terms of definition \ref{def:identity}.
\end{theorem}
\begin{proof}
Please refer to the Appendix \ref{proof:unibi can model} 
\end{proof}

\subsection{Efficient Modeling}
\label{sec:efficient modeling}
\subsubsection{Efficient Modeling for Spectral Radius}
Although the proposed model has been proven to model the law of identity, it still has a practical disadvantage, since it is difficult to directly represent all matrices whose spectral raidus are $1$. In addition, it is also time-consuming to calculate the spectral radius $\rho(\cdot)$ via singular values decomposition (SVD). 

To avoid unnecessary decomposition, we divide a relation matrix into three parts $\mathbf{R} = \mathbf{R}_h\mathbf{\Sigma R}_t$ where $\mathbf{R}_h, \mathbf{R}_t$ are orthogonal matrices and $\mathbf{\Sigma} = \text{Diag}[\sigma_1, \dots, \sigma_n]$ is a positive semidefinite diagonal matrix. And we maintain the independence of these three components during training. Therefore, it becomes simple to obtain matrices whose spectral radius is 1, that is, $\frac{\mathbf{R}_h\mathbf{\Sigma}\mathbf{R}_t}{\sigma_{max}}$. And we transform the score function Eq. \ref{equ:unibi:constrain} into the following form.
\begin{equation}
\label{equ:svd}
    s(h,r,t) = \frac{\mathbf{h}^\top \mathbf{R}_h\mathbf{\Sigma R}_t \mathbf{t}}{\sigma_{max}\|\mathbf{h}\|\|\mathbf{t}\|},
\end{equation}
where $\sigma_{max}$ is the maximum among $\sigma_i$.


\subsubsection{Efficient Modeling for Orthogonal Matrix}
In addition, we find that the calculation of the orthogonal matrix is still time-consuming~\citep{tang-etal-2020-orthogonal}. To this end, we only consider the diagonal orthogonal block matrix, where each block is a low-dimensional orthogonal matrix. Specifically, we use $k$-dimensional rotation matrices to build $\mathbf{R}_h$ and $\mathbf{R}_t$. Taking $\mathbf{R}_h$ as an example $\mathbf{R}_h = \text{Diag}[\mathbf{SO}(k)_1, \dots, \mathbf{SO}(k)_\frac{n}{k}]$, where $\mathbf{SO}(k)_i$ denotes the $i$-th special orthogonal matrix, that is, the rotation matrix. 

The rotation matrix only represents the orthogonal matrices whose determinant are $1$ and does not represent the ones whose determinant are $-1$. To this end, we introduce two diagonal sign matrices of $n$-th order $\mathbf{S}_h, \mathbf{S}_t \in \mathbb{S}$ where
\begin{equation}
     \mathbb{S} = \{\mathbf{S}\ |\  \mathbf{S}_{ij} = \begin{cases} \pm 1, \  &if\  i = j, \\ 0, \  &if \  i \neq j.\end{cases}\}.
\end{equation}

Thus, we could rewrite the score function Eq. \ref{equ:svd} to
\begin{equation}
\label{equ:svd with sign}
s(h,r,t) = \frac{\mathbf{h^\top R}_h{\mathbf{S}_h\mathbf{\Sigma} \mathbf{S}_t\mathbf{R}_t \mathbf{t}}}{\sigma_{max}\|\mathbf{h}\|\|\mathbf{t}\|}.
\end{equation}

However, the sign matrix $\mathbf{S}_h$ and $\mathbf{S}_t$ are discrete. To address this problem, we notice that $\mathbf{S}_h$, $\mathbf{\Sigma}$, $\mathbf{S}_t$ can be merged into a matrix $\mathbf{\Xi}$ that
\begin{equation}
\label{equ:xi}
\mathbf{\Xi}_{ij} = \begin{cases} s_is_j\sigma_i, \ &if \ i = j, \\ 0 \ &if\  i \neq j.\end{cases}
\end{equation}
where $s_i = (\mathbf{S}_h)_{ii}$, $s_j = (\mathbf{S}_t)_{jj}$, $i,j = 1, \dots, n$ and $\mathbf{\Xi} = \text{Diag}[\xi_1,\dots, \xi_n]$. Thus, we incorporate the discrete matrices $\mathbf{S}_h,\mathbf{S}_t$ into the continuous matrix $\mathbf{\Xi}$.
\begin{equation}
\label{equ:final}
s(h,r,t) = \frac{\mathbf{h}^\top \mathbf{R}_h\mathbf{\Xi R}_t \mathbf{t}}{|\xi_{max}|\|\mathbf{h}\|\|\mathbf{t}\|},
\end{equation}
where $|\xi_{max}|$ is the maximum among $|\xi_i|$. 


\subsection{Other benefits from scale normalization}
\label{sec:properties}
In addition to theoretical superiority, UniBi is more powerful and interpretable, since modeling the law of identity requires normalizing the scales that barely contain any useful knowledge. 
On the one hand, it is obvious that modeling the law of identity needs to avoid the cases in Fig. \ref{fig:case 1} and Fig. \ref{fig:case 2}, which requires normalizing the scales of entities and relations. 
On the other hand, it is counter-intuitive that the scale information is useless for bilinear based models. 

Scale information is treated as useless because what really matters is not the absolute values but the relative ranks of scores. And scale contributes nothing to the ranks, since they remain the same after we multiply these scores by a positive factor:
\begin{align}
s'(h,r,t) = (k_e\mathbf{h})^\top (k_r\mathbf{R})(k_e\mathbf{t}) = k_e^2k_r(\mathbf{h^\top Rt}) = k_e^2k_r\cdot s(h,r,t),
\end{align}
where $k_e,k_r > 0$. Therefore, we treat learning on scales as ineffective\footnote{Further discussion in Appendix \ref{app:ineffective learning}.}.

\subsubsection{Performance}
As illustrated in Fig. \ref{fig:identity scale performane}, UniBi has better performance, since it prevents ineffective learning with the least constraints. On the one hand, by preventing ineffective learning, UniBi focuses more on learning useful knowledge, which helps improve performance. On the other hand, it pays a negligible cost of expressiveness, since it adds only one equality constraint to each entity or relation, which is ignorable when the dimension $n$ is high.

In other words, although our scale normalization is a double-edged sword, its negative effect is negligible, and thus leads to a better performance. It should be noticed that the loss on expressiveness may outweighs the gain on learning, if scale normalization is replaced by a sticker one. For example, if we constrain the matrix to be orthogonal, the cost of expressiveness is no longer negligible, since an orthogonal matrix requires that each of its singular values be $1$, which is $n$ equality constraints.

\subsubsection{Interpretability}
\label{sec:ambiguity}

\begin{figure}[t]
    \centering
    \subfigure[]{
        \includegraphics[width=8cm]{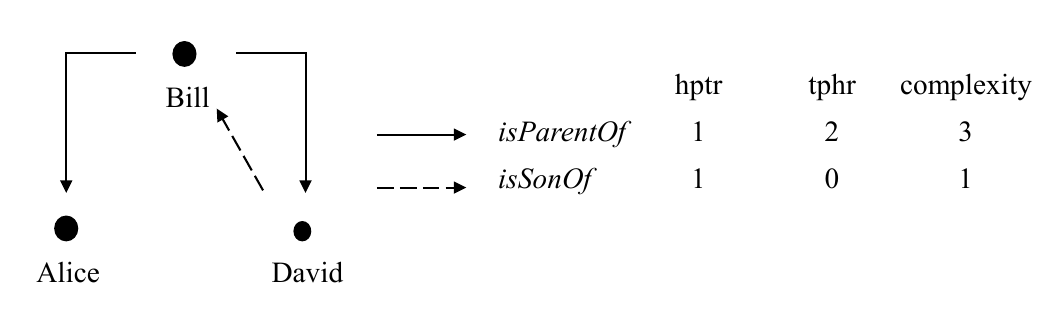}
        \label{fig:complexity case}
    }
    \subfigure[]{
        \includegraphics[width=4cm]{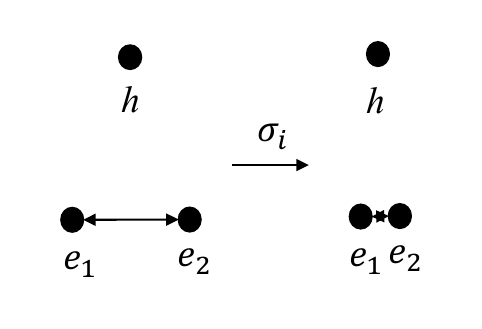}
        \label{fig:aggregation}
    }
    \caption{Complexity and contraction. (a) A toy example to show how to calculate complexity. (2) the aggregation corresponds to the singular values less than 1.}
\end{figure}





In addition to performance, scale normalization also helps us to understand complex relations. Complex relations are defined by whither hptr (\textit{h}ead \textit{p}er \textit{t}ail of a \textit{r}elation) or tphr (\textit{t}ail \textit{p}er \textit{h}ead of a \textit{r}elation) is higher than a special threshold 1.5~\citep{transh}. Therefore, all relations are divided into 4 types, i.e. 1-1, 1-N, N-1, and N-N. However, we think this division is too coarse-grained and suggest a fine-grained continuing metric complexity instead. To better demonstrate this idea, we gave an example in Fig. \ref{fig:complexity case} and the definition of complexity as follows.

\begin{definition}
\label{def:complexity}
The complexity of a relation is the sum of its hptr and tphr.
\end{definition}





Intuitively, complex relations are handled by aggregating entities through projection~\citep{transh,transr}, which implies that the higher the complexity of a relation, the stronger its aggregation effect, and vice versa.

We note that this aggregation effect can be well characterized by the relative ratio, or imbalance degree, of singular values of the matrices of relations. For any relation matrix $\mathbf{R} = \mathbf{U\Sigma V^\top}$, both $\mathbf{U}$ and $\mathbf{V}$ are isometry, and only the singular values of the scaling matrix $\mathbf{\Sigma}$ contribute to the aggregation.
Moreover, the singular values of UniBi are less than or equal to 1\footnote{We discuss this characteristic in more depth from the perspective of group in Appendix \ref{app:monoid}.}, since the spectral radius, i.e. the maximum singular value, are normalized. This shows a promising correspondence between the singular values of our model and the aggregation and further to the complexity, as demonstrated in Fig. \ref{fig:aggregation}.
Therefore, we could use singular values to represent the complexity of relations, which increases the interpretability of UniBi.

It is worth mentioning that this interpretabity can be transferred to other bilinear based models if they normalize the spectral radius of their relation matrix as UniBi does.

\section{Experiment}
\label{sec:experiment}
In this section, we give the experiment settings in Section \ref{sec:exp setting}. We verify that UniBi is capable to model the law of identity, while previous bilinear based models are not in Section \ref{sec:model identity}. UniBi is comparable to previous SOTA bilinear models in the link prediction task, as shown in Section \ref{sec:main results}. In addition, we demonstrate the robustness of UniBi in Section \ref{sec:robust} and the interpretability about complexity in Section \ref{sec:ambiguity analysis}.

\begin{table}[h]
\centering
\caption{Statistics of the benchmark datasets.}
\begin{tabular}{ccccccc}
    \toprule
    Dataset & $|\mathcal{E}|$ & $|\mathcal{R}|$ & Training & Validation & Test \\
    \midrule
    WN18RR &  40,943 & 11 & 86,835 & 3,034 & 3,134 \\
    FB15k-237 & 14,541 & 237 & 272,115 & 17,535 & 20,466 \\
    YAGO3-10-DR & 122,873 & 36 & 732,556 & 3,390 & 3,359 \\
    \bottomrule
\end{tabular}
\label{tab:statistics}
\end{table}

\subsection{Experiment Settings}
\label{sec:exp setting}

\textbf{Dataset} \hspace{1mm} We evaluate models on three commonly used benchmarks, i.e. WN18RR~\citep{conve}, FB15k-237~\citep{fb237} and YAGO3-10-DR~\citep{reeval}. They are proposed by removing the reciprocal triples that cause data leakage in WN18, FB15K and YAGO3-10 respectively. Their statistics are listed in Tbl. \ref{tab:statistics}.

\textbf{Evaluation metrics} \hspace{1mm} We use Mean Reciprocal Rank (MRR) and Hits@k (k = 1, 3, 10) as the evaluation metrics. MRR is the average inverse rank of the correct entities that are insensitive to outliers. Hits@k denotes the proportion of correct entities ranked above k. 

\textbf{Baselines}\hspace{1mm} Here we consider two specific versions  of UniBi: UniBi-O(2), UniBi-O(3), which use rotation matrices in $2$ and $3$ dimensions to construct the orthogonal matrix. To be specific, we use the unit complex value and the unit quaternion to model the 2D and 3D rotations using the $2\times 2$ and $4\times4$ matrices, respectively. For more details, see the Appendix \ref{app:rotation matrix}.

UniBi is compared with these bilinear models: RESCAL~\citep{rescal}, CP~\citep{cp}, ComplEx~\citep{complex}, and QuatE~\citep{quate}. In addition, it also compared to other models: RotatE~\citep{sun2018rotate}, MurE~\citep{murp} and RotE~\citep{rote}, Turcker~\citep{balavzevic2019tucker} and ConvE~\citep{conve}, PairRE~\cite{pairre}, TripleRE. 

\textbf{Optimization}\hspace{1mm} We adopt the reciprocal setting~\citep{n3}, which creates a reciprocal relation $r'$ for each $r$ and a new triple $(e_k, r'_j, e_i)$ for each $(e_i, r_j, e_k) \in \mathcal{K}$. Instead of using Cross Entropy directly~\citep{n3,dura,quate}, we add an extra scalar $\gamma > 0$ before softmax function. Since UniBi is bounded, it brings an upper bound to loss that makes the model difficult to optimize as discussed by~\citet{normface}.
\begin{equation}
\label{equ:loss function}
L = -\sum_{(h,r,t)\in \mathcal{K}_{train}}\log\left({\frac{\exp(\gamma\cdot s(h,r,t))}{\sum_{t' \in \mathcal{E}}{\exp(\gamma \cdot s(h, r, t'))}}} \right)+ \lambda \cdot Reg(h,r,t),
\end{equation}
where $Reg(h,r,t)$ is the regularization term and $\lambda > 0$ is its factor. Specifically, we only take $Reg(h,r,t)$ as DURA~\citep{dura} in experiments, since it significantly outperforms other regularization terms. In addition, $\gamma$ is set to $1$ for previous methods and greater than $1$ for UniBi. And we set the dimension $n$ to $500$. For other details on the implementation, see Appendix \ref{app: hyperparameters}. 
\subsection{Modeling Prior Property}
\label{sec:model identity}
In this part, we verify that 1) UniBi is capable to model the law of identity while previous models fail, and 2) both constraints on the embedding of entities and relations are indispensable. We explicitly add \textit{identity} as a new relation to benchmarks and use its corresponding matrix to determine whether the uniqueness is modeled. In particular, since entities are different, modeling the law means model the uniqueness of \textit{identity}, which requires the matrix of \textit{identity} relation is supposed to converge to the identity matrix $\mathbf{I}$ or a scaled one. To evaluate it, we introduce a new metric imbalance degree $\Delta = (\sum_i \frac{\sigma_i}{\sigma_{max}} - 1)^2$.

We first compare UniBi with CP~\citep{cp} and RESCAL~\citep{rescal}, the least and most expressive bilinear model on FB15k-237. Besides, we also apply DURA~\citep{dura} to models to explore whether these methods are able to model the law of identity under extra regularization. As demonstrated in Fig. \ref{fig:identity is harmful for other models}, the imbalance degree $\Delta$ of UniBi converges
to $0$ while others fails, which verifies that UniBi is capable to uniquely model \textit{identity}. In addition, the imbalance of other models decreases to some extent when using DURA, yet they are still unable to uniquely model \textit{identity}. Then, to show that UniBi can uniquely converge to \textit{identity}, we use two matrices $\mathbf{R}_1$ and $\mathbf{R}_2$ to model it independently. As shown in Fig. \ref{fig:converge different datasets}, the error between $\mathbf{R}_1$ and $\mathbf{R}_2$ also converges to $0$, which means that they all converge to $\mathbf{I}$.

We then perform an ablation study to verify that both the entity constraint (EC) and the relation constraint (RC) are needed to model \textit{identity} uniquely. The experiments show that only using either constraint is not enough to model the uniqueness of \textit{identity}, as illustrated in Fig. \ref{fig:both regular are indispensable}. And this verify the existence of problems shown in Fig. \ref{fig:case 1} and Fig. \ref{fig:case 2}.

\subsection{Main Results}
\label{sec:main results}

\begin{figure}[t]
    \centering
    \subfigure[UniBi V.S. others.]{
        \includegraphics[width=4.3cm]{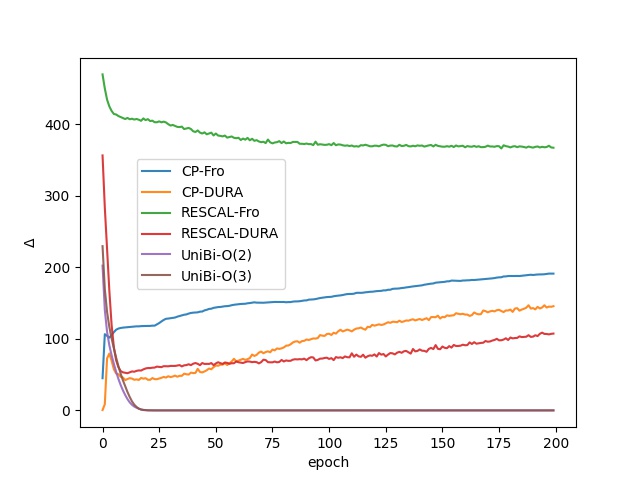}
        \label{fig:identity is harmful for other models}
    }
    \subfigure[Error between different matrices converges.]{
        \includegraphics[width=4.3cm]{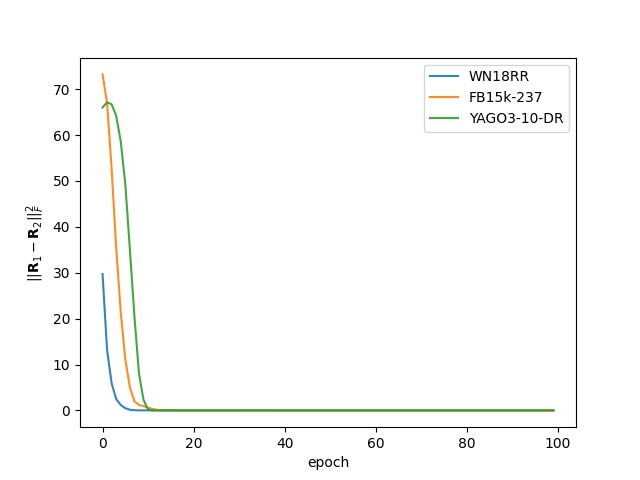}
        \label{fig:converge different datasets}
    }
    \subfigure[Abaltion of constrains.]{
        \includegraphics[width=4.3cm]{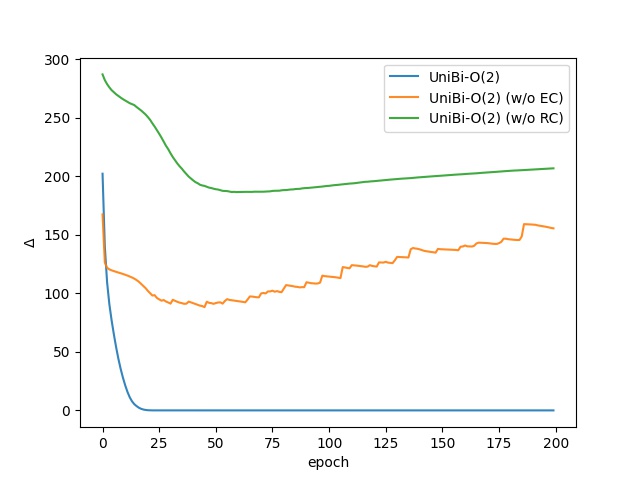}
        \label{fig:both regular are indispensable}
    }

    \caption{UniBi is capable to uniquely model \textit{identity}. (a) the imbalance degree ($\Delta$) of UniBi converges to $0$ while others diverge. (b) The errors between different matrices modeling \textit{identity} converge to $0$ on different datasets.  (c) Both entity constrain (EC) and relation constrain (RC) are indispensable for UniBi to model \textit{identity}.}
\end{figure}    

\begin{table}[t]
    \centering
    \caption{Evaluation results on WN18RR, FB15k-237 and YAGO3-10-DR datasets. We reimplement RotE, CP, RESCAL, ComplEx with $n=500$ and denoted by \dag, while we take results on WN18RR and FB15k-237 from the origin papers and YAGO3-10-DR from~\citet{reeval}. Best results are in \textbf{bold} while the seconds are \underline{underlined}.}
    \resizebox{\textwidth}{!}{
    \begin{tabular}{lccccccccc}
    \toprule
    \multicolumn{1}{c}{} & \multicolumn{3}{c}{\textbf{WN18RR}}& \multicolumn{3}{c}{\textbf{FB15K-237}}& \multicolumn{3}{c}{\textbf{YAGO3-10-DR}} \\
    Model & MRR & Hits@1  & Hits@10 & MRR & Hits@1  & Hits@10 & MRR & Hits@1 & Hits@10  \\ \midrule

DistMult & 0.43 & 0.39 & 0.49 & 0.241 & 0.155 & 0.419 & 0.192 & 0.133 &  0.307\\ 
ConvE & 0.43 & 0.40 &  0.52 & 0.325 & 0.237 & 0.501 & 0.204 & 0.147 & 0.315 \\
TuckER & 0.470 & 0.443 &  0.526 & 0.358 & 0.266 &  0.544 & 0.207 & 0.148 &  0.320\\
QuatE & 0.488 & 0.438 &  0.582 & 0.348 & 0.248 &  0.550 &  - & -& -\\
RotatE & 0.476  & 0.428  & 0.571 & 0.338 & 0.241  & 0.533 & 0.214 & 0.153 &  0.332\\
MurP & 0.481 & 0.440 &  0.566 & 0.335 & 0.243  & 0.518 & - & - & -\\
RotE & \underline{0.494} & 0.446  & \textbf{0.585} & 0.346 & 0.251 & 0.538 & - & -  & - \\
PairRe & - & - & - & 0.351 & 0.256 & 0.544 & - & - & - \\
TripleRE & - & - & - & 0.351 & 0.251 & 0.552 & - & -& -\\
CP\dag & $0.457_{\pm 0.001}$ & $0.414_{\pm 0.002}$  & $0.549_{\pm 0.003}$ & $0.361_{\pm 0.001}$ & $0.266_{\pm 0.001}$ & $0.551_{\pm 0.001}$ & $0.241_{\pm 0.001}$ & $0.175_{\pm 0.002}$ & $0.370_{\pm 0.003}$\\
ComplEx\dag & $0.487_{\pm 0.002}$  & $0.445_{\pm 0.001}$ & $\underline{0.571}_{\pm 0.002}$ & $0.363_{\pm 0.001}$ & $0.269_{\pm 0.001}$  & $0.552_{\pm 0.002}$ & $0.238_{\pm 0.001}$ & $0.174_{\pm 0.002}$  & $0.360_{\pm 0.004}$\\
RESCAL\dag &$\textbf{0.495}_{\pm 0.001}$ & $\textbf{0.452}_{\pm 0.002}$ &  $0.575_{\pm 0.002}$ &  $0.364_{\pm 0.003}$ & $\underline{0.272}_{\pm 0.003}$ &  $0.547_{\pm 0.002}$ & $0.233_{\pm 0.003}$ & $0.168_{\pm 0.004}$ & $0.360_{\pm 0.004}$\\

\midrule
\rowcolor{lightgray} UniBi-O(2) & $0.487_{\pm 0.002}$ & $0.46_{\pm 0.002}$ & $0.566_{\pm 0.002}$ & $\textbf{0.370}_{\pm 0.001}$ & $\textbf{0.274}_{\pm 0.001}$ & $\textbf{0.561}_{\pm 0.002}$ & $\textbf{0.247}_{\pm 0.001}$ & $\underline{0.179}_{\pm 0.002}$ &  $\underline{0.376}_{\pm 0.002}$\\
\quad  - w/o constraint & $0.488_{\pm 0.001}$ & $\underline{0.447}_{\pm 0.002}$ & $0.568_{\pm 0.003}$ & $0.361_{\pm 0.001}$ & $0.267_{\pm 0.001}$ & $0.550_{\pm 0.001}$ & $0.242_{\pm 0.001}$ & $0.176_{\pm 0.001}$ & $0.372_{\pm 0.002}$ \\ 
\rowcolor{lightgray} UniBi-O(3)  & $0.492_{\pm 0.001}$ & $\textbf{0.452}_{\pm 0.001}$ & $\underline{0.571}_{\pm 0.001}$ & $\underline{0.369}_{\pm 0.001}$ & $\textbf{0.274}_{\pm 0.001}$ & $\underline{0.561}_{\pm 0.001}$ & $\underline{0.246}_{\pm 0.001}$ & $\textbf{0.180}_{\pm 0.001}$ & $\textbf{ 0.377}_{\pm 0.001}$\\ 

\quad - w/o constraint & $0.488_{\pm 0.001}$ & $0.446_{\pm 0.002}$ & $0.567_{\pm 0.003}$ & $0.361_{\pm 0.001}$ & $0.265_{\pm 0.001}$ & $0.550_{\pm 0.001}$ & $0.241_{\pm 0.001}$ & $0.175_{\pm 0.002}$ &  $0.370_{\pm 0.003}$\\ 
\bottomrule
    \end{tabular}}
    \label{tab:main table}
\end{table}
In this part, we demonstrate that the constraints helps UniBi to achieve better performance. We mainly compared our model with previous SOTA models, i.e. CP~\citep{cp}, ComplEx~\citep{complex} and RESCAL~\citep{rescal} using DURA regularization. Although these models have been implemented by~\citet{dura}, the dimensions of CP and ComplEx are very high and have not been tested on YAGO3-10-DR, so we reimplement them in this paper. In addition, we further remove the constraint of UniBi-O(n) as ablations to eliminate the influence of other factors.

In Tbl. \ref{tab:main table}, UniBi achieves comparable results to previous bilinear based models and the unconstrained versions. UniBi is only slightly and justifiably below RESCAL on WN18RR, since RESCAL needs require much more time and space\footnote{Detailed in Appendix \ref{app:time and space}.}. 

\subsection{UniBi prevents ineffective learning}
\label{sec:robust}


\begin{figure}[t]
    \centering
    \subfigure[Ablation of regularization.]{
        \includegraphics[width=6.7cm]{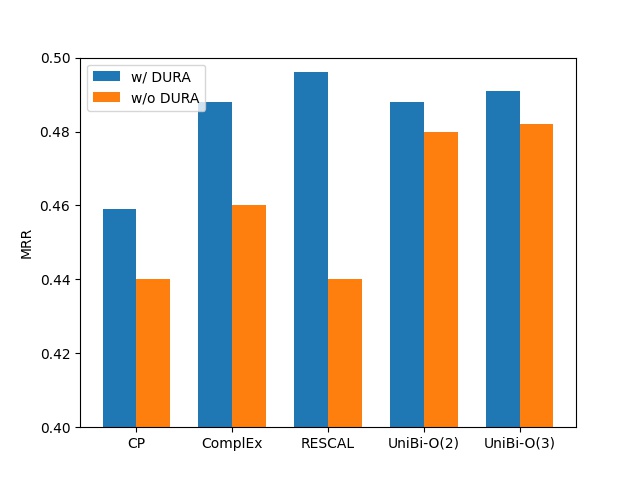}
        \label{fig:without dura}
        }
    \subfigure[Ablation of constrains.]{
        \includegraphics[width=6.7cm]{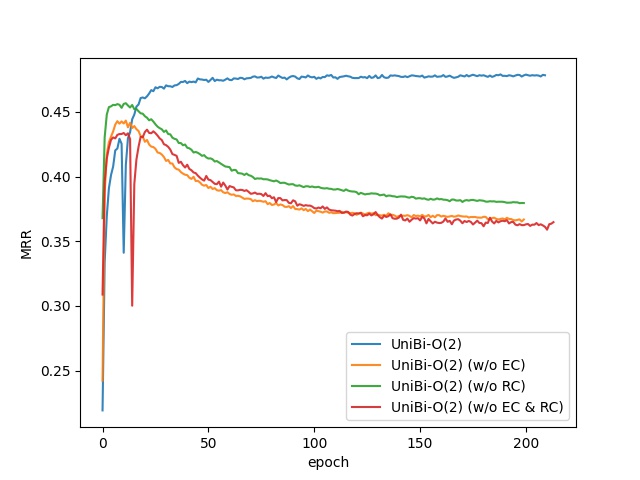}
        \label{fig:valid mrr}
        }
        
        \caption{UniBi benefits from preventing ineffective learning. (a) UniBi less relies on regularization and other models are not. (b) Neither entity constraint (EC) nor relation constrain (RC) alone stops the sliding of performance.}
\end{figure}

In this part, we verify the superiority of UniBi comes from preventing ineffective learning. We conduct further comparisons without regularization. In addition, we also adopt EC and RC in Section \ref{sec:model identity} to study the effect of both constraints. All experiments are implemented on WN18RR.

On the one hand, UniBi decreases slightly, while others decrease significantly when the regularization term is removed, as demonstrated in Fig. \ref{fig:without dura}. It shows that learning of UniBi is less dependent on extra regularization, since it is better at learning by preventing the ineffective part. On the other hand, we illustrate the MRR metric of UniBi and its ablation models on validation set as epoch grows in Fig. \ref{fig:valid mrr}. It shows that either constrain alleviates overfitting to some extend but fails to prevent the downward sliding behind since the scale of the unconstrained part may diverge. Thus, both constraints are verified to be indispensable for preventing the ineffective learning and improving performance of UniBi.



\subsection{Correlation to Complexity }
\label{sec:ambiguity analysis}
To verify the statement in Section \ref{sec:ambiguity}, we study the connection between singular values and the complexity of each relation on three benchmarks, where complexity is calculated following definition \ref{def:complexity}. Furthermore, we measure the singular values of a relation by the imbalance degree $\Delta$. To differentiate $\Delta$ of a relation $r$ and its reciprocal relation $r'$, we use $\Delta_r$ and $\Delta_{r'}$ to denote them.

As demonstrated in Fig. \ref{fig:delta datasets}, we find that singular values are highly correlated with the complexity of a relation. Furthermore, we notice that $\Delta_r$ and $\Delta_{r'}$ are very close even if a relation is unbalanced (1-N or N-1), which shows that complexity is handled by aggregation regardless of direction.

\begin{figure}[t]
    \centering
    \subfigure[WN18RR]{
        \includegraphics[width=4.3cm]{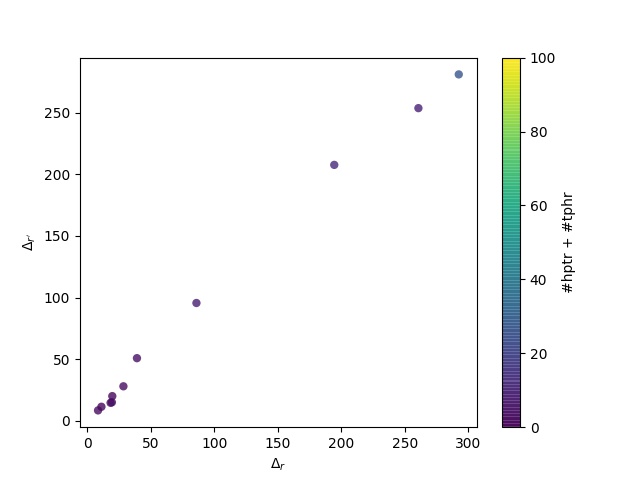}
        \label{fig:delta wn18rr}
    }
    \subfigure[FB15k-237]{
        \includegraphics[width=4.3cm]{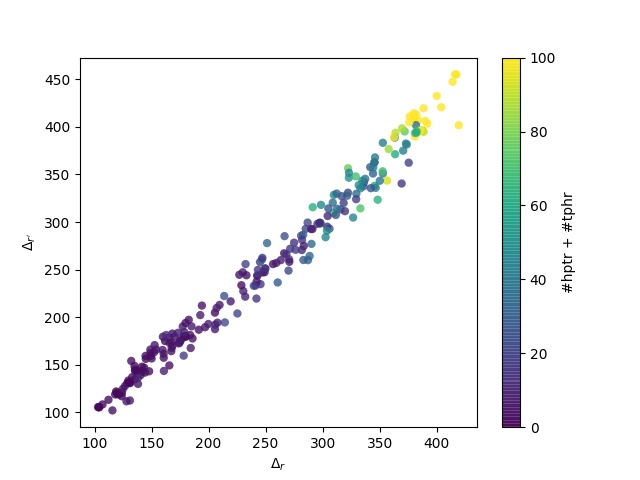}
        \label{fig:delta fb237}
    }
    \subfigure[YAGO3-10-DR]{
        \includegraphics[width=4.3cm]{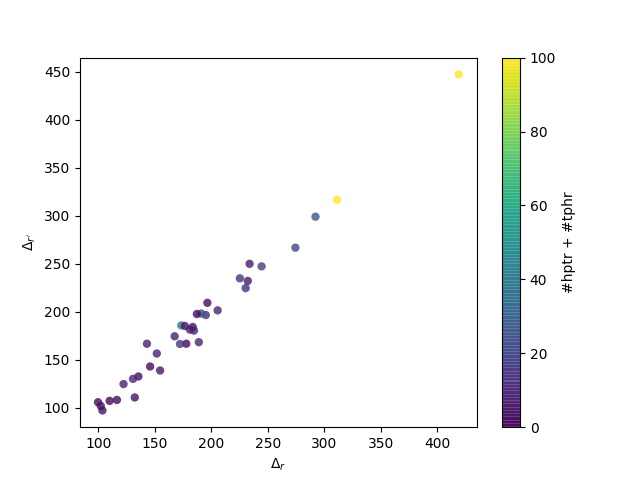}
        \label{fig:delta yago}
    }
    \caption{The imbalance degree ($\Delta$) and complexity ($\#$ hptr + $\#$ tphr) of relations in WN18RR, FB15k-237 and YAGO3-10-DR respectively. Two metric are highly correlated and the imbalance of a relation ($\Delta_r$) and the imbalance of its reciprocal one ($\Delta_{r'}$) are very close.}
    \label{fig:delta datasets}
\end{figure}


\section{Conclusion}
In this paper, we propose a new perspective, prior property, on analysis and modeling KGs beyond posterior properties. We discover a prior property named the law of identity. Moreover, We notice that bilinear based models fail to model this property and propose a proved solution named UniBi. Specifically, UniBi applies well-designed normalization on the embedding of entities and relations with minimal constraints. 

Furthermore, UniBi gains other advantages through the normalization. On the one hand, the normalization prevents ineffective learning and leads to better performance; on the other hand, it reveals that the relative ratio of singular values corresponds to the complexity of relations and improves interpretability.

In summary, we believe the question of prior property and the paradigm of UniBi can provide interesting and useful directions for the studies of bilinear based models.







\subsection{Retrieval of style files}

The style files for NeurIPS and other conference information are available on
the World Wide Web at
\begin{center}
  \url{http://www.neurips.cc/}
\end{center}
The file \verb+neurips_2022.pdf+ contains these instructions and illustrates the
various formatting requirements your NeurIPS paper must satisfy.

The only supported style file for NeurIPS 2022 is \verb+neurips_2022.sty+,
rewritten for \LaTeXe{}.  \textbf{Previous style files for \LaTeX{} 2.09,
  Microsoft Word, and RTF are no longer supported!}

The \LaTeX{} style file contains three optional arguments: \verb+final+, which
creates a camera-ready copy, \verb+preprint+, which creates a preprint for
submission to, e.g., arXiv, and \verb+nonatbib+, which will not load the
\verb+natbib+ package for you in case of package clash.

\paragraph{Preprint option}
If you wish to post a preprint of your work online, e.g., on arXiv, using the
NeurIPS style, please use the \verb+preprint+ option. This will create a
nonanonymized version of your work with the text ``Preprint. Work in progress.''
in the footer. This version may be distributed as you see fit. Please \textbf{do
  not} use the \verb+final+ option, which should \textbf{only} be used for
papers accepted to NeurIPS.

At submission time, please omit the \verb+final+ and \verb+preprint+
options. This will anonymize your submission and add line numbers to aid
review. Please do \emph{not} refer to these line numbers in your paper as they
will be removed during generation of camera-ready copies.

The file \verb+neurips_2022.tex+ may be used as a ``shell'' for writing your
paper. All you have to do is replace the author, title, abstract, and text of
the paper with your own.

\bibliographystyle{plainnat}
\bibliography{kge}

\appendix
\newpage

\section*{Appendix}

\section{Proofs}


\subsection{UniBi Is Bounded}
\begin{proposition}
\label{prop:bound}
UniBi is bounded that $s(h,r,t) \in \left[-1, 1\right]$.
\end{proposition}
\begin{proof}
By the Cauchy-Schwarz inequality and the fact that for a vector L2 norm is equivalent to spectral norm, we have the following.
\begin{equation}
    \|\mathbf{h^\top Rt}\| \leq \|\mathbf{h^\top R}\|\|\mathbf{t}\| \leq \rho(\mathbf{R})\|\mathbf{h}\|\|\mathbf{t}\|=1.
\end{equation}
\end{proof}

\subsection{Proof of Theorem \ref{theorem: can model}}
\label{proof:unibi can model}
\begin{proof}

On the one hand, if $\mathbf{R} = \mathbf{I}$, it is easy to have $\forall \mathbf{h,t}\in\mathbb{\hat{E}}, \mathbf{h} \neq \mathbf{t}$ that $\mathbf{h^\top I h} > \mathbf{h^\top I t}$ from the property of cosine. 

On the other hand, $\forall \mathbf{R}\in\mathbb{\hat{R}}, \mathbf{R} \neq \mathbf{I}$, we can always give a counterexample. Using singular value decomposition (SVD), we have $\mathbf{R} = \mathbf{U}\mathbf{\Sigma}\mathbf{V^\top}$, where $\Sigma = \text{Diag}[\sigma_1,\dots, \sigma_n]$ with $\sigma_i\geq0$ and $U,V$ are orthogonal matrices. Since $\rho(\mathbf{R}) = 1$, we have $\sigma_{max} = \max(\sigma_i)$ = 1.

Besides, we notice that since $U,V$ are orthogonal matrices that do not change the norm of vectors, we have $\|\mathbf{h^\top U}\| = \|\mathbf{V^\top t}\| = 1$ and use $\mathbf{\hat{h}}$ and $\mathbf{\hat{t}}$ to denote $\mathbf{U^\top h}$ and $\mathbf{V^\top t}$ for mathematical simplicity. We consider three scenarios and discuss them separately.

(1) If all singular values of $\mathbf{R}$ are equivalent, we have $\mathbf{\Sigma} = \mathbf{I}$, and we have:

\begin{equation}
    s(h,r,t) = \mathbf{h^\top U \Sigma V ^\top t} = \mathbf{\hat{h}^\top I\hat{t}} = \mathbf{\hat{h}^\top\hat{t}}.
\end{equation}

It is easy to notice that the above equation just goes back to the cosine function, and it has the maximum value when $\mathbf{\hat{h}} = \mathbf{\hat{t}}$ and $\mathbf{U^\top h} = \mathbf{V^\top t}$. If $\mathbf{U V^\top} = \mathbf{I}$, this contradicts the assumption that $\mathbf{R}\neq\mathbf{I}$. If $\mathbf{U V^\top} \neq \mathbf{I}$, then we have $\mathbf{h}\neq \mathbf{t}$, since $\mathbf{h} = \mathbf{(U^\top)}^{-1}\mathbf{V^\top t} = \mathbf{UV^\top t}$. It means $\mathbf{h^\top R h} < \mathbf{h^\top R t}$ in this situation.

(2) If not all singular values of $\mathbf{R}$ are equivalent and $\mathbf{U} = \mathbf{V}$, there $\exists i,j\in 1,\dots,n$ that $\sigma_i \neq \sigma_j$. It may be assumed that $\sigma_i > \sigma_j$. Then we take $\forall \mathbf{h} \in \mathbb{\hat{E}}$ that has $\mathbf{\hat{h}}_j = (\frac{\sigma_i}{\sigma_j} - \epsilon)\mathbf{\hat{h}}_i$ where $\epsilon \in (0, \frac{\sigma_i}{\sigma_j} - 1)$. Then we take 

\begin{equation}
    \mathbf{\hat{t}}_k = \begin{cases}\mathbf{\hat{h}}_k,\ &k\neq i,j, \\
    \mathbf{\hat{h}}_j,\ & k=i, \\
    \mathbf{\hat{h}}_t,\ &k=j.\end{cases}
\end{equation}
It is easy to notice that $\mathbf{\hat{h}} = \mathbf{U^\top h}, \mathbf{\hat{t}} = \mathbf{U^\top t}$ and $\mathbf{h} \neq \mathbf{t}$, then we have
\begin{equation}
\label{equ:proof theroem h r t expanded}
\begin{aligned}
    \mathbf{h^\top R t} & = \mathbf{h^\top U R U^\top t} \\
    & = \mathbf{\hat{h}^\top\Sigma\hat{t}} \\
    & = \sigma_i\mathbf{\hat{h}}_i\mathbf{\hat{t}}_i + \sigma_j\mathbf{\hat{h}}_j\mathbf{\hat{t}}_j + \sum_{k\neq i,j}\sigma_k\mathbf{\hat{h}}_k\mathbf{\hat{t}}_k \\
    & = \sigma_i\mathbf{\hat{h}}_i\mathbf{\hat{h}}_j + \sigma_j\mathbf{\hat{h}}_j\mathbf{\hat{h}}_i+  \sum_{k\neq i,j}\sigma_k\mathbf{\hat{h}}_k^2,
\end{aligned}
\end{equation}
similarly, we have
\begin{equation}
\label{equ:proof theroem h r h expanded}
    \begin{aligned}
        \mathbf{h^\top R h} & = \mathbf{h^\top U R U^\top h} \\
    & = \mathbf{\hat{h}^\top\Sigma\hat{h}} \\
    & = \sigma_i\mathbf{\hat{h}}_i^2 + \sigma_j\mathbf{\hat{h}}_j^2 + \sum_{k\neq i,j}\sigma_k\mathbf{\hat{h}}_k^2.
    \end{aligned}
\end{equation}
Then, we use Eq. \ref{equ:proof theroem h r h expanded} minus Eq. \ref{equ:proof theroem h r t expanded}, and we have
\begin{equation}
    \begin{aligned}
        &\mathbf{h^\top R h} - \mathbf{h^\top R t} \\ 
         = & \left(\sigma_i\mathbf{\hat{h}}_i^2 + \sigma_j\mathbf{\hat{h}}_j^2 + \sum_{k\neq i,j}\sigma_k\mathbf{\hat{h}}_k^2\right) - \left(\sigma_i\mathbf{\hat{h}}_i\mathbf{\hat{h}}_j + \sigma_j\mathbf{\hat{h}}_j\mathbf{\hat{h}}_i +  \sum_{k\neq i,j}\sigma_k\mathbf{\hat{h}}_k^2 \right) \\
         = & \sigma_i(\mathbf{\hat{h}}_i^2 - \mathbf{\hat{h}}_i\mathbf{\hat{h}}_j) + \sigma_j(\mathbf{\hat{h}}_j^2 - \mathbf{\hat{h}}_i\mathbf{\hat{h}}_j) \\
         = & \sigma_i\mathbf{\hat{h}}_i(\mathbf{\hat{h}}_i - \mathbf{\hat{h}}_j) - \sigma_j\mathbf{\hat{h}}_j(\mathbf{\hat{h}}_i - \mathbf{\hat{h}}_j) \\
         = & (\sigma_i\mathbf{\hat{h}}_i - \sigma_j\mathbf{\hat{h}}_j)(\mathbf{\hat{h}}_i - \mathbf{\hat{h}}_j) \\
         = & \left(\sigma_i\mathbf{\hat{h}}_i - \sigma_j (\frac{\sigma_i}{\sigma_j} - \epsilon)\mathbf{\hat{h}}_i\right)\left(\mathbf{\hat{h}}_i -  (\frac{\sigma_i}{\sigma_j} - \epsilon)\mathbf{\hat{h}}_i\right) \\
         = & \left(\sigma_i\mathbf{\hat{h}}_i - \sigma_i\mathbf{\hat{h}}_i + \epsilon\sigma_j\mathbf{\hat{h}}_i\right)\left(1-\frac{\sigma_i}{\sigma_j} + \epsilon \right)\mathbf{\hat{h}}_i \\
         = & \epsilon\sigma_j\mathbf{\hat{h}}_i^2(1-\frac{\sigma_i}{\sigma_j} + \epsilon) \\
         < & 0,
    \end{aligned}
\end{equation}
which means that $\mathbf{h^\top Rh} < \mathbf{h^\top Rt}$ in this case.

(3) If not all singular values of $\mathbf{R}$ are equivalent and $\mathbf{U} \neq \mathbf{V}$, there exists $k \in 1,\dots,n$ such that $\sigma_k = \sigma_{max} =  1$ since $\rho(\mathbf{R}) = 1$. Then we take the following.
\begin{equation}
    \mathbf{\hat{h}}_i = \mathbf{\hat{t}}_i = \begin{cases} 1,\ & i = k, \\ 0,\ &i\neq k.\end{cases}
\end{equation}
It should be noted that $\mathbf{\hat{h}} = \mathbf{\hat{t}}$ and $\mathbf{h} \neq \mathbf{t}$ since $\mathbf{U} \neq \mathbf{V}$. Then, we have
\begin{equation}
    \begin{aligned}
    \mathbf{h^\top Rt} & = \mathbf{h^\top U\Sigma V^\top t}\\
    & = \mathbf{\hat{h}^\top\Sigma \hat{t}} \\
    & = \sigma_k \\
    & = 1.
    \end{aligned}
\end{equation}
Since Proposition \ref{prop:bound} proves that UniBi is bounded by $[-1,1]$, we have $\mathbf{h^\top R h} \leq 1 = \mathbf{h^\top R t}$, which means that $s(h,r,h) > s(h,r,t)$ does not always hold.

In summary, we can conclude that UniBi has $\mathbf{h}\neq\mathbf{t}, \  \mathbf{h^\top R h} > \mathbf{h^\top R t}$ iff. $\mathbf{R} = \mathbf{I}$, which means that UniBi can model \textit{identity} uniquely. And considering that the embedding of entities are differences, we conclude that UniBi can model the law of identity in terms of definition~\ref{def:identity}

\end{proof}
\subsection{Proof of Theorem \ref{prop:identity constrain}}
\label{proof:necessary constraints}
\begin{proof}
if we have $\|\mathbf{e}\|= 1, \rho(\mathbf{R}) = 1$. The equation for the entity part obviously holds, and by Proposition \ref{prop:bound}, we have the following.
\begin{equation}
    \|\mathbf{Re}\| \leq \rho(\mathbf{R})\|\mathbf{e}\| = \rho(\mathbf{R})=1,
\end{equation}
similarly, we have $\|\mathbf{e^\top R}\| \leq 1$.

Moreover, such a condition will become necessary and sufficient if $\forall \mathbf{R}\in \mathbb{\hat{R}}$, $\exists \hat{e}$ that either $\|\mathbf{\hat{e}^\top R}\| = 1$ or $\|\mathbf{R\hat{e}}\| = 1$.

To prove it, if we have $\|\mathbf{e^\top R}\| \leq 1$ and $\|\mathbf{Rt}\| \leq 1$, we use SVD to any $\mathbf{R}$ and get $\mathbf{R} = \mathbf{U\Sigma V}$. Then we denote $\mathbf{\sigma} = Diag(\mathbf{\Sigma})$ and we have $\forall \mathbf{e}, \|\mathbf{\sigma \cdot e}\| \leq 1$.


If $\exists i, \sigma_i > 1$, we take $i \neq j, e_i = 1, e_j = 0$ to show that $\|\mathbf{e\cdot \sigma}\|$ is larger than 1. Therefore, we have $\forall i, \sigma_i \leq 1$. Moreover, if $\forall i, 0 < \sigma_i < 1$, we take $\mathbf{\bar{e}} = \mathbf{V\hat{e}}$
\begin{align}
    \|\mathbf{\sigma \cdot \bar{e}}\|^2 = \sum_i{\sigma_i^2\bar{e}_i^2} < \sigma_i \bar{e}_i^2 = 1,
\end{align}
which also violates $\|\mathbf{Re}\| = \|\mathbf{\Sigma V\hat{e}}\| = \|\mathbf{\sigma \cdot \bar{e}}\| = 1$. Therefore, there $\exists k, \sigma_k = 1$ and we have $\rho(\mathbf{R}) = 1$.

\end{proof}
\subsection{Proof of Proposition \ref{prop:not invertible}}
\label{app:proof:not invertible}

\begin{proof}

If a relation $r$ is invertible, it means there exists an inverse relation $r^{-1}$ that ensure $r \circ r^{-1} = r^{-1} \circ r = identity$, where $\circ$ is the composition of relations. Consider that \textit{identity} relation means an entity is itself, thus KGs only contain triples like $(e_1,identity,e_1)$ rather than $(e_1, identity, e_2)$, $\forall e_1, e_2\in\mathcal{E}$ and $e_1\neq e_2$.

Consider a complex relation $r$ that has multiple tail entities for a head entity, that is, $\exists h,t_1,t_2, t_1\neq t_2\ (h,r,t_1),(h,r,t_2)$. In order to remap the entities $t_1$ and $t_2$ back to $h$, the inverse relation $r^{-1}$ has to contain triples that $(t_1,r^{-1},h),(t_2,r^{-1},h)$. Although $r\circ r^{-1}$ could map $h$ to $h$, the other order $r^{-1}\circ r$ fails to map $t_1,t_2$ back to themselves separately since both $(t_1, r^{-1} \circ r, t_1)$ and $(t_1, r^{-1}\circ r, t_2)$ are true. Therefore $r$ is not invertible. Similarly, if a relation has multiple heads, we can also obtain the above conclusion.

Therefore, all complex relations are inherently non-invertible.

\end{proof}

\section{Background of Group and Monoid}
\label{app:background of group and monoid}
\begin{definition}[\cite{jacobson2012basic}]
\label{def:monoid}
A monoid is a triple $(M, p, 1)$ in which $M$ is a non-vacuous set, $p$ is an associative binary composition (or product) in $M$, and 1 is an element of $M$ such that $p(1,a) = a= p(a,1)$ for all $ a \in M$
\end{definition}

\begin{definition}[\cite{jacobson2012basic}]
\label{def:group}
A group $G$ ( or (G,p,1)) is a monoid all of whose elements are invertible.
\end{definition}

\section{Discuss of Distance Based models}
\label{app:proof:distance based models}
We mention that distance based model could model \textit{identity} uniquely, and here we give its corresponding proof. Here, we only consider models that can be written in the basic form of $s(h,r,t) = -\|\mathbf{Rh}-\mathbf{t}\|$. Note that translation is also considered in such a form if we take translation as a linear transformation. Furthermore, the model must ensure that $\mathbf{I} \in \mathbb{\hat{R}}$ and $\mathbb{\hat{E}} = \mathbb{R}^n$. Other peculiar scenarios are not considered in this discussion. 
\begin{theorem}
A distance based model $s(h,r,t) = -\|\mathbf{Rh}-\mathbf{t}\|$ with $\mathbf{I} \in \mathbb{\hat{R}}$ and $\mathbb{\hat{E}} = \mathbb{R}^n$ can model \textit{identity} uniquely. 
\end{theorem}
\begin{proof}
On the one hand, if $\mathbf{R} = \mathbf{I}$, it is easy to have
\begin{equation}
    -\|\mathbf{Rh} - \mathbf{t}\| = -\|\mathbf{h} - \mathbf{t}\| \leq -\|\mathbf{h} - \mathbf{h}\| = 0,
\end{equation}
where $\mathbf{h}\neq \mathbf{t}$.

On the other hand, if $\mathbf{R} \neq \mathbf{I}$ and $\mathbf{R}$ is not a singular matrix, we take $\mathbf{t} = \mathbf{Rh}$ and have
\begin{equation}
    -\|\mathbf{Rh} - \mathbf{h}\| \leq 0 = -\|\mathbf{Rh} - \mathbf{Rh}\| = -\|\mathbf{Rh} -\mathbf{t}\|.
\end{equation}

If if $\mathbf{R} \neq \mathbf{I}$ and $\mathbf{R}$ is a singular matrix, then $\mathbf{Rx} = 0$ has solutions that $\mathbf{x} \neq \mathbf{0}$, and by taking $\mathbf{t} = \mathbf{x} + \mathbf{h}$ we have:
\begin{equation}
    \mathbf{h^\top Rt} - \mathbf{h^\top Rh} = \mathbf{h^\top R(t - h)} = \mathbf{h^\top R x} = 0,
\end{equation}
which means $s(h,r,h) = s(h,r,t)$.

Therefore, we prove that distance-based models are born to uniquely model \textit{identity}.
\end{proof}

\section{Comparison to previous restrictions}
\label{app: comparison to previous restrictions}
Although it is true that entity normalization and set $\mathbf{R}$ to orthogonal matrices, which is a special case of $\rho(\mathbf{R}) = 1$, are common in KGE models, our constraint differs from their approach as follows. 

(1) The constraint for relation based on the spectral radius, i.e., $\rho(\mathbf{R}) = 1$, is first proposed in our work, which is indispensable and irreplaceable to model the \textit{identity} uniquely without the cost of the performance. If we set $\mathbf{R}$ to be orthogonal while keeping the normalization of entities, the performance will drop significantly, e.g., MRR: 0.488 $\to$ 0.471 for UniBi-O(2) on WN18RR. Since orthogonal matrix requires each singular value to be $1$, which means $n$ equality constraints and the cost on expressiveness is no longer negligible.

(2) Only the combination of entity and relation constraints succeeds to model \textit{identity} uniquely, as shown in Figure \ref{fig:both regular are indispensable}. It suggests that constraints on entity and relation should be treated as a whole rather than a combination of two unrelated things.

(3) At first glance, our constraint looks very similar to that of TransR~\citep{transr}, but in fact there is a big difference. The constraint of TransR is $\|\mathbf{h}\| \leq 1$, $\|\mathbf{t}\| \leq 1$, $\|\mathbf{hM}_r\|\leq 1$, and $\|\mathbf{tM}_r\|\leq 1$, and has three differences from ours. 
\begin{enumerate}[i)]
    \item For TransR, $\|\mathbf{hM}_r\|\leq 1$ and $\|\mathbf{tM}_r\| \leq 1$ is the constraint \textbf{itself}. In contrast, $\|\mathbf{R^\top h}\|\leq 1$ and $\|\mathbf{Rt}\| \leq 1$ is \textbf{deduction} of our constraints $\|\mathbf{e}\| = 1$ and $\rho(\mathbf{R}) = 1$.
    \item TransR does not normalize the entities, and we have shown in Fig. \ref{fig:case 1} that normalization is necessary for a bilinear based method to uniquely model \textit{identity}.
    \item TransR is a distance based model, and as we have shown in Appendix \ref{app:proof:distance based models}, distance based models do not need to consider the problem of \textit{identity}, thus the proposal of TransR is difference to us. 
\end{enumerate}

Therefore, we believe that our combination of constraints is novel, since it proposes a new constraint for the relation and the two parts are deliberately rather than arbitrarily combined.

\begin{figure}[t]
    \centering
    \subfigure[RESCAL w/o reg.]{
        \includegraphics[width=4cm]{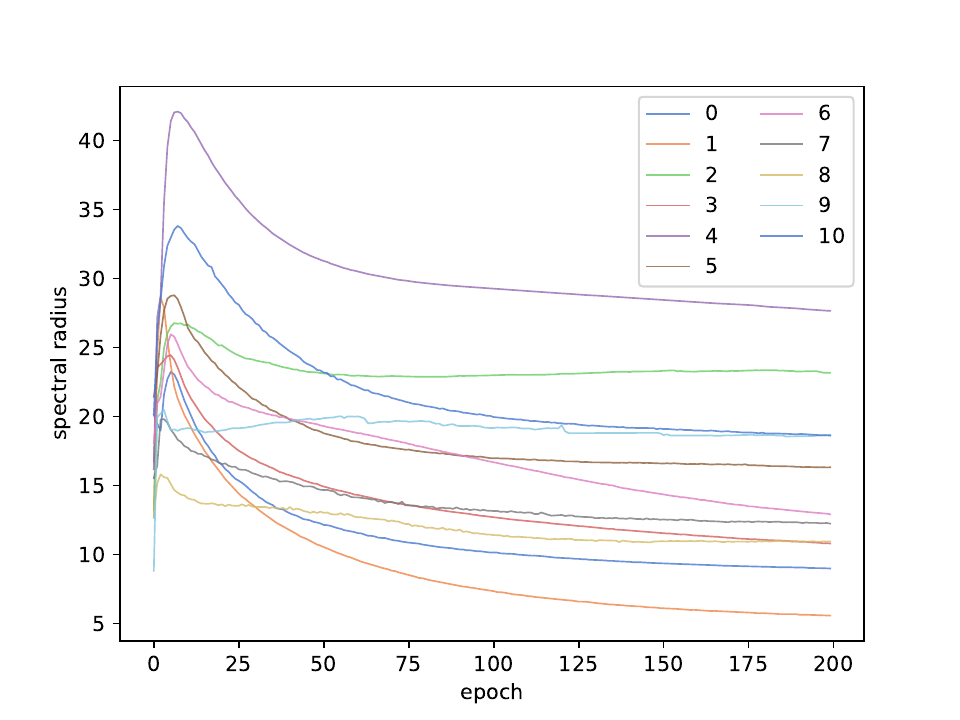}
        \label{fig:spectral:rescal}
    }
    \subfigure[RESCAL w/ Frobenius]{
        \includegraphics[width=4cm]{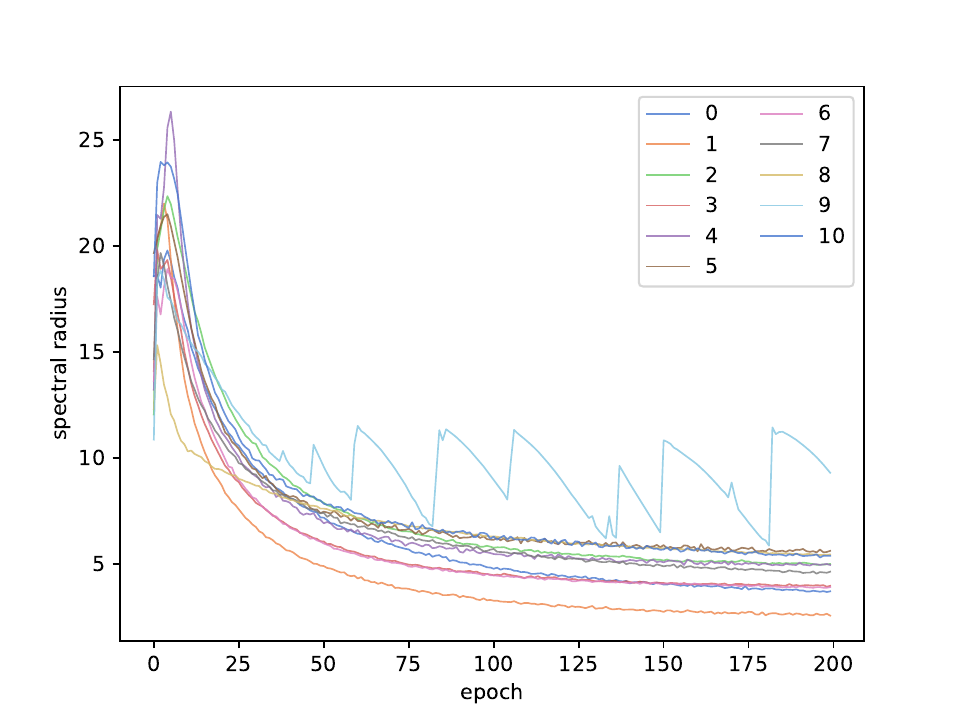}
        \label{fig:spectral:rescal fro}
    }
    \subfigure[RESCAL w/ DURA]{
        \includegraphics[width=4cm]{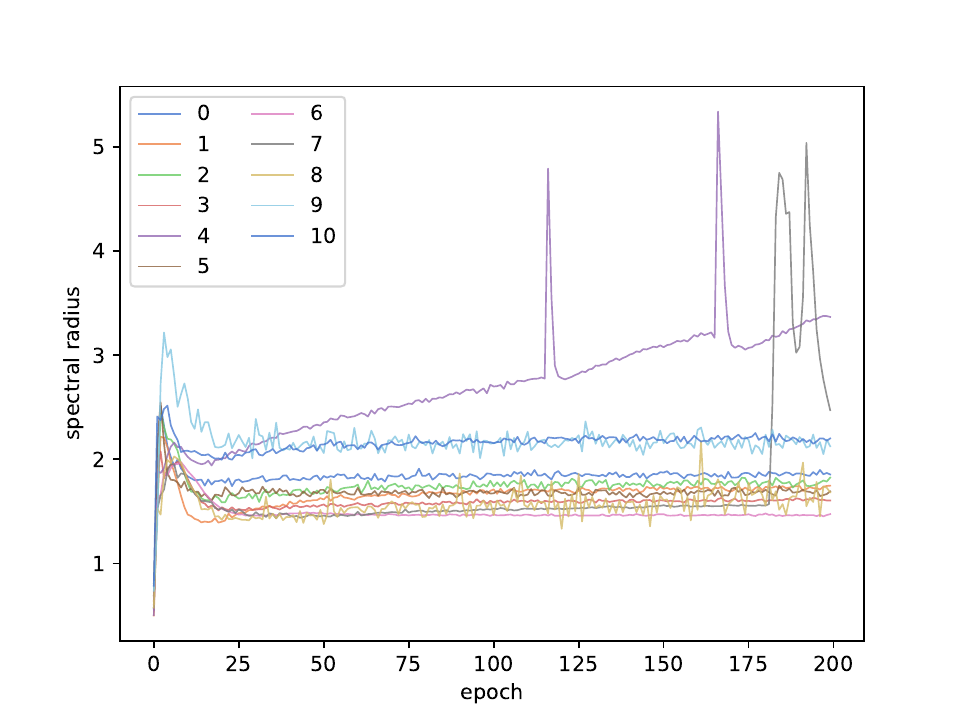}
        \label{fig:spectral:rescal dura}
    }
    \caption{Why learning on the scales are ineffective. We test RESCAL with 3 setting a) no regularization, b) use Frobenius norm as regularization, c) use DURA as regularization. We use index rather than name to denote different relations for simplicity. And we notice that 1) non-convergence exists in every case, 2) the better the result, the more the scales converge, 3) Regularization cannot stop the fluctuation of scales. (Better view in color, zoom in, note the difference in the vertical coordinates.)}
    \label{fig:spectral:rescal-all}
\end{figure}

\section{Ineffective learning}
\label{app:ineffective learning}

Here, we demonstrate that ineffective learning does exist, which means the scale is not only redundant but also harmful. As shown in Fig. \ref{fig:spectral:rescal-all}, we take RESCAL as example to show this phenomenon. And we notice that 1) non-convergence exists in every case, 2) the better the result, the more the scales converge, and 3) regularization cannot stop the fluctuation of scales.

We believe that these cases illustrate, on the one hand, the positive correlation between the constraint scale and the effect, on the other hand, the mere constraint cannot eliminate fluctuations that may interfere with the model learning. Therefore, we think scale is harmful and learning on it is ineffective, and we need a hard constraint rather than regularization term to prevent this completely. 

\section{From Group to Monoid}
\label{app:monoid}
Since the singular values are less than or equal to 1, some matrices do not have their corresponding inverse elements in $\mathbb{\hat{R}}$. Therefore, UniBi violates the definition of group, which requires that each element has its inverse element inside. However, we argue that some relations are inherently non-invertible and misidentified by previous work~\citep{NagE,sun2018rotate, dihedral}.

For example, in Fig. \ref{fig:non-invertible},
if two relations are inverse elements of each other, their composition must be \textit{identity}. However, we find the combination of \textit{IsParentOf} and \textit{IsChildOf} is not the \textit{identity}, since the child of Alice's parent is "Alice or David" rather than "Alice itself" as demonstrated in Fig. \ref{fig:true case} and Fig. \ref{fig:false case}.

It should be noted that we are not saying that these two relations are irrelevant, but rather pointing out that their combination does not strictly fit our definition of \textit{identity}. Moreover, the fact that they are both \textbf{non-invertible} does not mean that they cannot form the \textbf{inversion pattern}~\citep{sun2018rotate}, since definitions of the two concepts are different and do not conflict.

Beyond the specific example, we generalize it to all complex relations that appear in two difference triples $(e_1, r, e_2)$ and $(e_1, r, e_3)$ (or ($e_3, r, e_2$)), where $e_1, e_2, e_3 \in \mathcal{E}$, and $e_1 \neq e_2 \neq e_3$. Furthermore, we have the following proposition.

\begin{proposition}
\label{prop:not invertible}
$\forall r \in \mathcal{R}$, if $r$ is a complex relation, then $r$ is non-invertible.
\end{proposition}
\begin{proof}
Please refer to Appendix \ref{app:proof:not invertible}.
\end{proof}

Since these non-invertible relations extensively present in KGs, the requirement that each element has its inverse element in the set should be abandoned. Thus, rather than group, we think monoid is a better structure for relations by their definitions\footnote{Definition \ref{def:monoid} and Definition \ref{def:group}.}.

\begin{figure}[t]
    \centering
    \subfigure[]{
        \includegraphics[width=4cm]{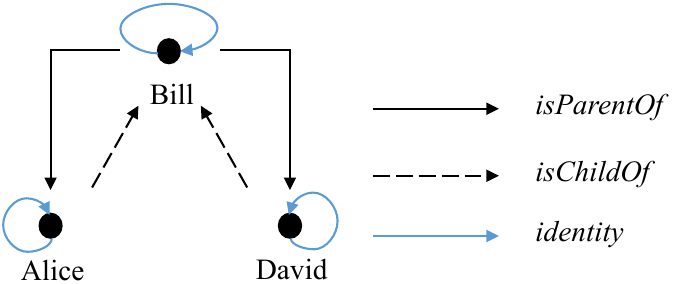}
        \label{fig:non-invertible}
    }
    \subfigure[]{
        \includegraphics[width=4cm]{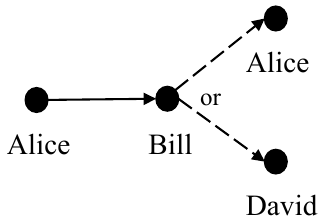}
        \label{fig:true case}
    }
    \subfigure[]{
        \includegraphics[width=4cm]{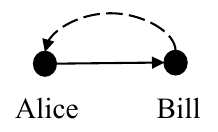}
        \label{fig:false case}
    }
    \caption{Why some relations are non-invertible. (a) A toy example. (2) True case. (3) False case.}
\end{figure}

\section{Implement Details}
\label{app:implement details}
\subsection{Rotation Matrices of UniBi}
\label{app:rotation matrix}
In Section \ref{sec:experiment}, we propose two variants of UniBi, i.e., UniBi-O(2) and UniBi-O(3). Here, we give the rotation matrix $\mathbf{SO}(k)$ they used specifically. 
\begin{equation}
    \mathbf{SO}(2) = \begin{bmatrix}
    x & -y \\
    y & x 
    \end{bmatrix},\  \mathbf{SO}(3) = \begin{bmatrix} p & -q & -u & -v\\
    q & p & v & -u \\
    u & -v & p & q\\
    v & u & -q & p\end{bmatrix},
\end{equation}
where $x,y, p,q,u,v\in\mathbb{R}$ and $x^2 + y^2 = 1, p^2 + q^2 + u^2 + v^2 = 1$.

\subsection{Hyperparameters}
\label{app: hyperparameters}
We fix the dimension of all models except RESCAL on WN18RR to $500$, while RESCAL on WN18RR is set to $256$ following~\citet{dura}. We choose Adam~\citep{adam} as the optimizer and fix the learning rate at $1e-3$. We set the maximum epochs to $200$ and apply the early stopping strategy.

We set the scaling factor $\gamma$ to $1$ for all models except UniBi. And we search $\gamma$ from $\left\{1, 5, 10, 15, 20, 25, 30\right\}$ for UniBi. For the factor for regularization $\lambda$ we search $\{1,5e-1,1e-1,5e-2,1e-2,5e-3,1e-3\}$ for all models except UniBi and $\{0.5, 1, 1.5, 2, 2.5,3\}$ for it. Furthermore, we do not search for $\lambda_1$ and $\lambda_2$ in Eq. \ref{equ:dura like} as in the original paper of DURA~\citep{dura}. we adopt their settings and set $\lambda_1 = 0.5, \lambda_2 = 1.5$ for ComplEx~\citep{complex} and CP~\citep{cp}, while $\lambda_1=\lambda_2=1$ for other cases. The search results are listed on the Tbl. \ref{tab:hyper para}. We search for the batch size from $\{100, 1000\}$. 

In addition, we implemented all the experiments in PyTorch with a single NVIDIA GeForce RTX 1080Ti graphics card. We repeat each experiment five times and take their means and standard deviations.

\begin{table}[!t]
    \centering
    \caption{Hyperparameters found by grid search. $\lambda$ is the regularization coefficient, $\gamma$ is the scaling factor, $b$ is the batch size.}
    \begin{tabular}{lccccccccc}
    \toprule
    \multicolumn{1}{c}{} & \multicolumn{3}{c}{\textbf{WN18RR}}& \multicolumn{3}{c}{\textbf{FB15K237}}& \multicolumn{3}{c}{\textbf{YAGO3-10-DR}} \\
    Model & $\lambda$ & $\gamma$ & b & $\lambda$ & $\gamma$ & b &  $\lambda$ & $\gamma$ & b  \\ 
    \midrule
    CP & 1e-1 & 1 & 100 & 5e-2 & 1 & 100  & 5e-3 & 1 & 1000\\
    ComplEx & 1e-1 & 1 & 100 & 5e-2 & 1 & 100 & 1e-2 & 1 & 1000\\
    RESCAL & 1e-1 & 1 & 1000 & 5e-2 & 1 & 1000 & 5e-2 & 1 &1000\\
    \midrule 
    UniBi-O(2) & 2 & 20 & 100 & 2 & 25 & 1000 & 1.5 & 30 & 1000 \\
    \quad - w/o constraint & 1e-1 & 1 & 100 & 5e-2 & 1 & 1000 & 5e-2 & 1 & 1000\\
    UniBi-O(3) & 2 & 15 & 100 & 1.5 & 20 & 1000 & 1.5 & 30 & 1000\\
    
    \quad - w/o constraint & 1e-1 & 1 & 100 & 5e-2 & 1 & 1000 & 5e-2 & 1 & 1000\\
\bottomrule
    \end{tabular}
    \label{tab:hyper para}
\end{table}

For the experiments in Section \ref{sec:robust}, we set $\gamma$ for UniBi without DURA in Section \ref{sec:robust} to $10,15,25$ for WN18RR, FB15k-237, and YAGO3-10-DR for UniBi, respectively.


\section{Time and Space Complexity}
\label{app:time and space}
Since UniBi needs to calculate some constraints, it spends more space and time. Here, we compare the time and space consumption of UniBi and CP, ComplEx, and RESCAL on FB15k-237, and set the batch size to 1000.

As demonstrated in Fig. \ref{fig:compare time} and Fig. \ref{fig:compare space}, UniBi takes a little more time than CP and ComplEx, and more time to compute the regularization since it is more complex. Nevertheless, we noticed that UniBi occupies a space similar to that of CP and ComplEx. 

In addition, we note that in Section\ref{sec:main results}, UniBi outperforms all bilinear based models except RESCAL on WN18RR. And RESCAL needs significantly more time and space than other models. Therefore, it is not as efficient as other models.

In summary, although UniBi takes a little longer, it is a good balance between complexity and performance.

In future work, we will consider finding other constraints on the relation instead of the spectral radius. 

\begin{figure}[t]
    \centering
    \subfigure[Comparison of time.]{
    \includegraphics[width=6.7cm]{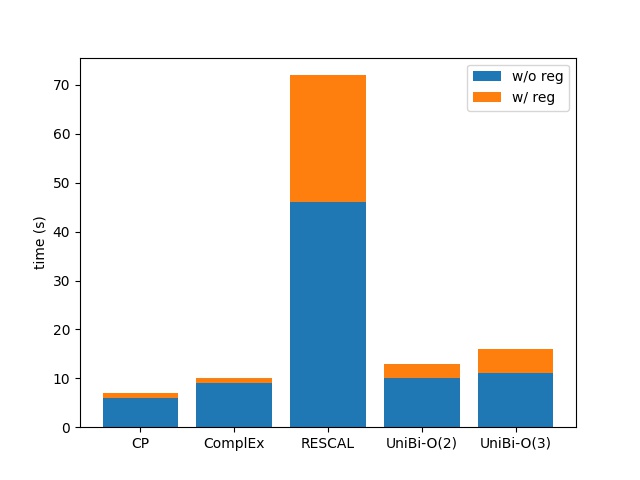}
    \label{fig:compare time}}
    \subfigure[Comparison of space.]{
    \includegraphics[width=6.7cm]{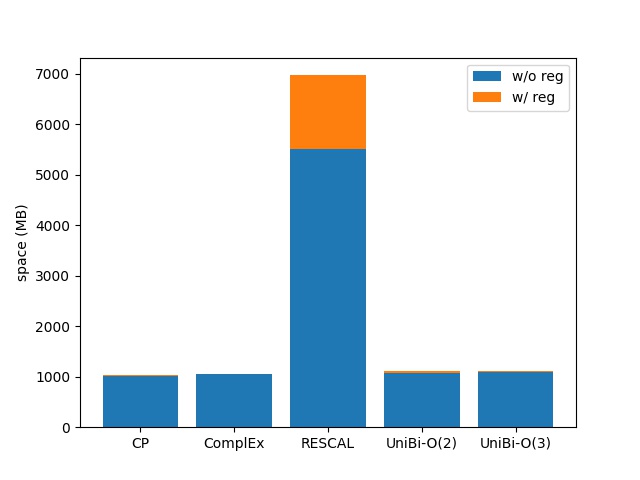}
    \label{fig:compare space}}
    \caption{UniBi takes a little more time than other models with linear parameters (CP and ComplEx) and significantly efficient than ones with quadratic parameters (RESCAL). }
    \label{fig:time space}
\end{figure}

\section{Inherent regularization term}
\label{sec:inherent regularization}

Here we find a necessary condition of UniBi and then deduce its corresponding Lagrangian function. 
\begin{theorem}
\label{prop:identity constrain}
UniBi has a necessary condition that $\|\mathbf{e}\| = 1$, $\|\mathbf{Re}\| \leq 1$ and $\|\mathbf{R^\top e}\| \leq 1$.
\end{theorem}
\begin{proof}
Please refer to the Appendix \ref{proof:necessary constraints}.
\end{proof}

If we ignore the other implicit constraints, the optimization of UniBi can be rewritten in the form of a constrained optimization problem.
\begin{equation}
\label{equ:constain}
    \begin{aligned}
    \min_{\mathbf{h,r,t}} \ & \sum_{(h_i,r_j,t_k)\in\mathcal{K}}f(h_i,r_j,t_k) \\
    s.t. \ & \|\mathbf{h}_i\|^2 = 1, \|\mathbf{t}_k\|^2 = 1\\
    &\|\mathbf{R}_j^\top \mathbf{h}_i\|^2 \leq 1, \|\mathbf{R}_j\mathbf{t}_k\|^2 \leq 1,
    \end{aligned}
\end{equation}
where $f(h,r,t)$ is a loss function. Furthermore, Eq. \ref{equ:constain} corresponds to a Lagrangian function.
\begin{equation}
\label{equ:lagrangian}
    \begin{aligned}
    \min_{\mathbf{h,r,t},\lambda, \mu}  \  &\sum_{(h_i,r_j,t_k)\in\mathcal{K}}f(h_i,r_j,t_k) +  \lambda_i^h(\|\mathbf{h}_i\|^2-1) +\lambda^t_k(\|\mathbf{t}_k\|^2 - 1) \\
    &+ \mu^h_j(\|\mathbf{R}_j^\top \mathbf{h}_i\|^2 - 1) + \mu^t_j(\|\mathbf{R}_j\mathbf{t}_k\|^2 - 1) \\
    s.t.\lambda_i^h,\lambda_k^t,\ &\mu_j^h,\mu_j^t \geq 0.
    \end{aligned}
\end{equation}

We notice that if we set factors $\lambda_i^h = \lambda_k^t = \lambda\lambda_1$ and $\mu_j^h = \mu_j^t = \lambda\lambda_2$ where $\lambda,\lambda_1,\lambda_2>0$, we achieve the following expression from Eq. \ref{equ:lagrangian} by discarding constant terms.

\begin{equation}
\label{equ:dura like}
    \begin{aligned}
    \min_{\mathbf{h,r,t}}  \  \sum_{(h_i,r_j,t_k)\in\mathcal{K}}f(h_i,r_j,t_k) +  \lambda\left[\lambda_1(\|\mathbf{h}_i\|^2 + \|\mathbf{t}_k\|^2) 
    + \lambda_2(\|\mathbf{R}_j^\top \mathbf{h}_i\|^2 + \|\mathbf{R}_j
    \mathbf{t}_k\|^2)\right],
    \end{aligned}
\end{equation}
which is equivalent to the optimization of a unconstrained model under DURA~\citep{dura}, the best regularization term for bilinear based model before. Therefore, we can get a more general version of DURA (DURA-G for simplicity) from Equ. \ref{equ:lagrangian}.
\begin{equation}
    \textbf{DURA-G}: \lambda_1\|\mathbf{h}_i\|^2 +\lambda_2\|\mathbf{t}_k\|^2+ \lambda_3\|\mathbf{R}_j^\top \mathbf{h}_i\|^2 + \lambda_4\|\mathbf{R}_j\mathbf{t}_k\|^2 
\end{equation}

At first glance, this seems to be nothing more than a worthless trick. However, in terms of DURA, DURA-G is nonsense and cannot deduce form its perspective of distance-based models (Please refer to Section 4.3 in the original paper for more details) while making sense in the perspective of Lagrangian function.

Although DURA-G has an additional hyperparameter to search and can lead to better results, we do not use this in experiments for three reasons: 1) for a more fair comparison with previous models, 2) the improvement is marginal, and 3) we prefer not to distract the reader from our key ideas.

{Extra experiments of Scales}
To demonstrate that the phenomenon in Fig. \ref{fig:case 2} happens, we carry out additional experiments on WN18RR by explicitly adding the identity relation explicitly and monitor its scales in different models. We run all experiments five times and take their means and variants. As shown in Fig. \ref{fig:identity scales}, we notice that all models except UniBi has fluctuation on the spectral radius of \textit{identity}, which verifies the phenomenon in Fig. \ref{fig:case 2}.

\begin{figure}[t]
\includegraphics[width=\textwidth]{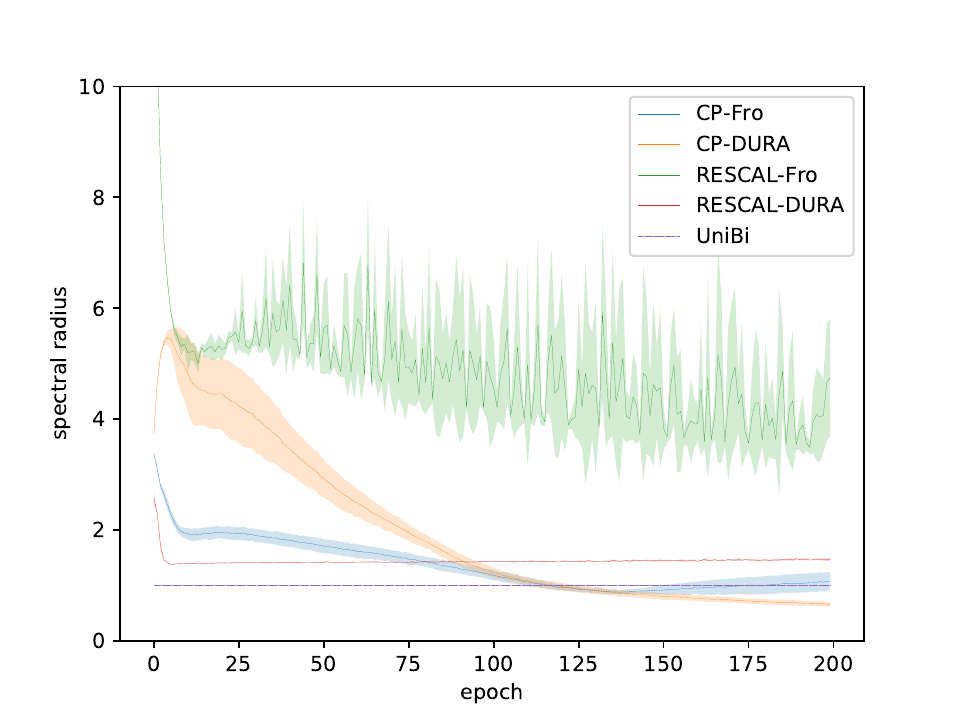}
\caption{Bilinear based models learning \textit{identity} have different scales. (Better zoom in to see the fluctuation of RESCAL-DURA.)}
\label{fig:identity scales}
\end{figure}

\begin{figure}
\begin{minipage}[b]{.5\linewidth}
\centering
\includegraphics[width=5cm]{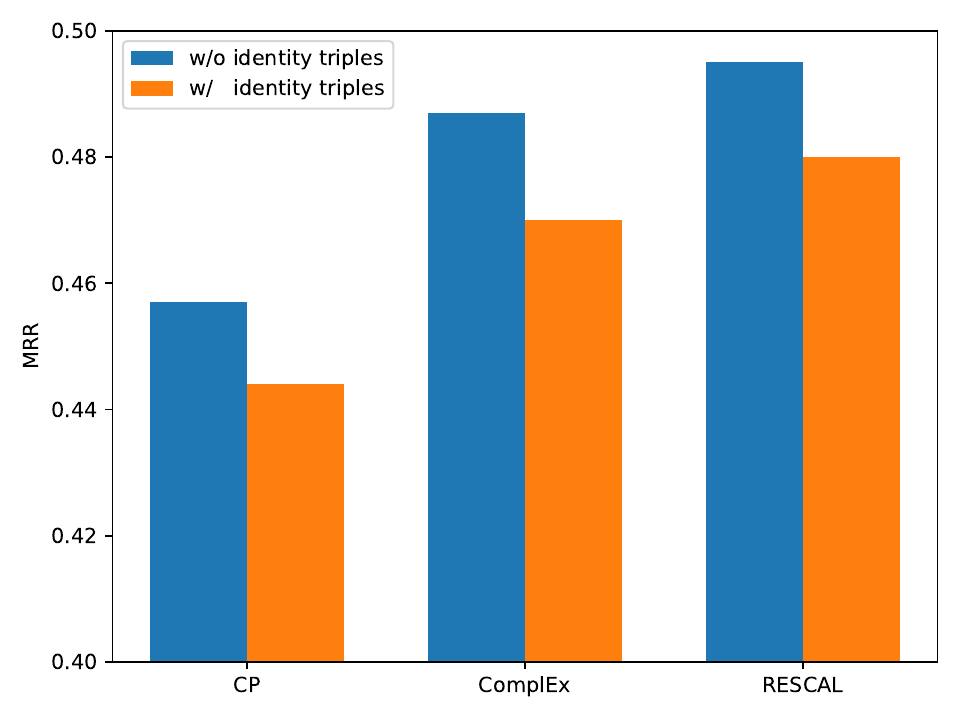}
\caption{Directly adding identity triples may hurt performance.}
\label{fig:add_identity}
\end{minipage}%
\begin{minipage}[b]{.5\linewidth}
\centering
\begin{tabular}{cccc}
\toprule
Dataset & \#Identity & \#Others(Avg.) \\
\midrule
WN18RR & 40,943 & 7,894 \\
FB15k-237 & 14,514 & 1,148\\
YAGO3-10-DR & 122,873 & 20,348\\
\bottomrule
\end{tabular}
\captionof{table}{Statistics of triples of identity and other relations in benchmark datasets.}
\label{tbl: statics of triples of identity and other relations}
\end{minipage}
\end{figure}

\section{Comparison of learning and modeling identity}

Readers may wonder why not add the identity relation to the training set and take this indirect approach. The reason is that learning on identity per se does not help the performance on other relations, and may be harmful for the overall results as shown in Fig .\ref{fig:add_identity}. We think the reason is that the triples of identity are too much compared to the ones of other relations as shown in Tbl. \ref{tbl: statics of triples of identity and other relations}, and the model negligence in learning these relations.

\section{Illustration of How the Performance is Improved}

This is a more detailed illustration of how the ineffective learning is prevented and enhance the performance as discussed in Section \ref{sec:properties}.

\begin{figure}[h]
    \centering
    \includegraphics[width=\textwidth]{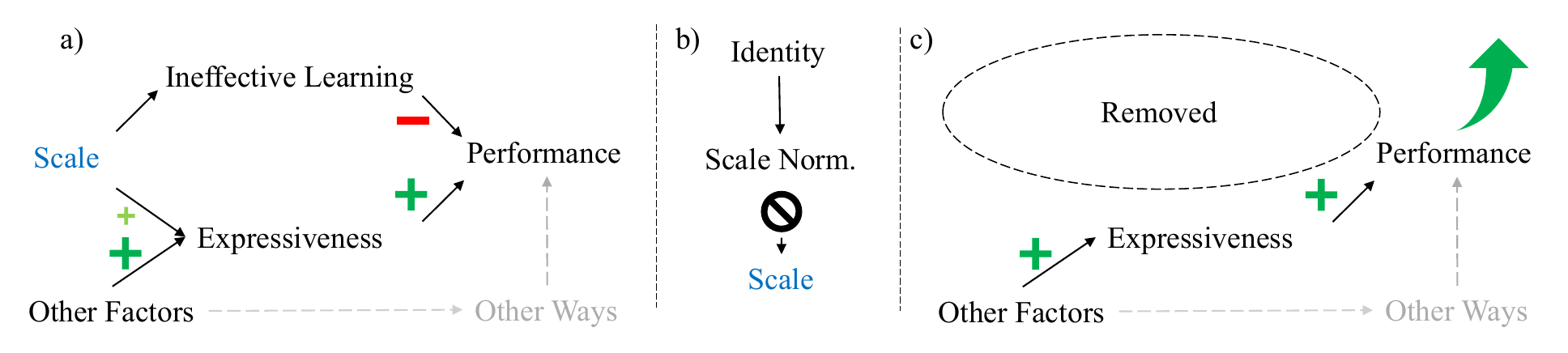}
    \caption{How modeling identity improves the performance. a) An illustration of how scale affects the performance. On the plus side, it has little effect on expressiveness, on the minus side, it causes the ineffective learning. b) Modeling identity requires scale normalization, which removes the effect of scales. c) UniBi improves the performance since it only deals with scales and not the other factors.}
    \label{fig:identity scale performane}
\end{figure}

\end{document}